\documentclass{article}

    \PassOptionsToPackage{numbers, sort&compress}{natbib}
\usepackage[preprint]{neurips_2026}

\usepackage[utf8]{inputenc} 
\usepackage[T1]{fontenc}    
\usepackage{hyperref}       
\usepackage{url}            
\usepackage{booktabs}       
\usepackage{amsfonts}       
\usepackage{nicefrac}       
\usepackage{microtype}      
\usepackage{xcolor}         
\usepackage{adjustbox}
\usepackage{amsmath}
\usepackage{amssymb}
\usepackage{mathtools}
\usepackage{amsthm}
\usepackage{bm}
\usepackage{graphicx}
\usepackage{float}
\usepackage{subcaption}
\usepackage{array}
\usepackage{multirow}
\usepackage{tcolorbox}
\usepackage{algorithm}
\usepackage{algorithmic}

\usepackage[capitalize,noabbrev]{cleveref}
\tcbuselibrary{skins,breakable}

\graphicspath{{fig/}}

\newcommand{\R}{\mathbb{R}}
\newcommand{\E}{\mathbb{E}}
\newcommand{\Ent}{\mathrm{H}}

\newcommand{\Tr}{\mathrm{Tr}}
\newcommand{\dd}{\mathrm{d}}
\newcommand{\vx}{\bm{x}}

\newcommand{\vv}{\bm{v}}
\newcommand{\vu}{\bm{u}}
\newcommand{\veps}{\bm{\epsilon}}

\newcommand{\cm}{\mathrm{cm}}
\newcommand{\bcr}{\mathrm{BCR}}
\newcommand{\diver}{\nabla\!\cdot}

\newcolumntype{P}[1]{>{\raggedright\arraybackslash}p{#1}}

\newcommand{\inlineeqtag}[1]{\refstepcounter{equation}\label{#1}\nobreak\hfill\textup{(\theequation)}\par}

\providecommand{\edmfidcell}[2]{%
  \mbox{#1\,{\tiny $\pm\,#2$}}%
}
\providecommand{\edmbestfidcell}[2]{%
  \begingroup\setlength{\fboxsep}{0.8pt}\colorbox{yellow!30}{\edmfidcell{#1}{#2}}\endgroup%
}
  \providecommand{\alphaplddtcell}[2]{%
    \mbox{#1\,{\fontsize{3.5pt}
  {5pt}\selectfont $\pm\,#2$}}%
  }
  \providecommand{\alphaplddtcellbold}[2]{%
    \mbox{\textbf{#1}\,{\fontsize{3.5pt}
  {5pt}\selectfont $\pm\,#2$}}%
  }
\providecommand{\alphabestplddtcell}[2]{%
  \begingroup\setlength{\fboxsep}{0.8pt}\colorbox{yellow!30}{\alphaplddtcell{#1}{#2}}\endgroup%
}
\providecommand{\alphabestplddtcellbold}[2]{%
  \begingroup\setlength{\fboxsep}{0.8pt}\colorbox{yellow!30}{\alphaplddtcellbold{#1}{#2}}\endgroup%
}

\newcolumntype{P}[1]{>{\raggedright\arraybackslash}p{#1}}

\theoremstyle{plain}
\newtheorem{theorem}{Theorem}
\newtheorem{proposition}{Proposition}

\theoremstyle{definition}
\newtheorem{assumption}{Assumption}
\theoremstyle{remark}

\newtcolorbox{computebox}{
  enhanced,
  breakable,
  colback=white,
  colframe=white,
  boxrule=0pt,
  borderline={0.5pt}{0pt}{gray!55,dotted},
  arc=1.5pt,
  left=5pt,
  right=5pt,
  top=4pt,
  bottom=4pt
}

\tcolorboxenvironment{proposition}{
  enhanced,
  breakable,
  colback=gray!7,
  colframe=gray!45,
  boxrule=0.5pt,
  arc=1.5pt,
  left=5pt,
  right=5pt,
  top=4pt,
  bottom=4pt
}

\tcolorboxenvironment{assumption}{
  enhanced,
  breakable,
  colback=white,
  colframe=black,
  boxrule=0.35pt,
  arc=2pt,
  borderline west={2pt}{0pt}{gray!45},
  left=6pt,
  right=5pt,
  top=4pt,
  bottom=4pt
}

\title{Entropy Across the Bridge: Conditional--Marginal Discretization for Flow and Schr\"odinger Samplers}

\author{%
  \normalfont\mdseries
  \textbf{Bruno Trentini}$^{1,2}$\thanks{Correspondence to: \texttt{brunod@nvidia.com}, \texttt{luca.ambrogioni@donders.ru.nl}} \qquad
  \textbf{Dejan Stancevic}$^{3}$ \qquad
  \textbf{Michael M.~Bronstein}$^{2,4}$ \\[0.5em]
  \textbf{Alexander Tong}$^{4}$ \qquad
  \textbf{Luca Ambrogioni}$^{3*}$ \\[1em]
  \small
  $^{1}$NVIDIA Corporation, Santa Clara, CA, USA \\
  $^{2}$University of Oxford, Dept.\ of Computer Science, Oxford, UK \\
  $^{3}$Donders Institute for Brain, Cognition, and Behaviour, Radboud University, Nijmegen, NL \\
  $^{4}$AITHYRA, Research Institute for Biomedical AI, Vienna, AT%
}

\begin{document}

\maketitle

\begin{abstract}
    For a fixed flow-based generative model under a small inference budget, sample quality can depend strongly on where the sampler spends its few function evaluations. Flow matching and Schr\"odinger bridges define probability paths, yet their inference grids are usually heuristic or inherited from one-endpoint diffusion. We derive a conditional--marginal entropy-rate objective for bridge-aware discretization, separating endpoint-conditioned bridge geometry from marginal flow evolution, and use it to build a training-free entropic inference-time scheduler from first principles. For Gaussian Brownian bridges this rate is closed-form and U-shaped, motivating boundary-heavy nonuniform grids. On trained two-dimensional bridge/flow models, the estimated profile recovers the predicted shape and improves 10-step ODE-Heun MMD over linear by 18.1\%, with a paired 22.7\% SDE-Heun improvement in the same low-NFE sweep. On EDM/CIFAR-10, the entropic time-discretization gives the best tested five-step FID ($186.3\pm4.0$ versus $200.5\pm2.9$ for linear and $238.0\pm5.3$ for cosine). On AlphaFlow protein generation, entropic conditional--marginal (cond-marg) scheduling shows advantage in low-NFE regimes on both CAMEO22 and ATLAS benchmarks. These results support entropy-rate scheduling as a practical low-budget allocation signal for high-dimensional bridge and flow samplers.

\end{abstract}

\section{Introduction}
\label{sec:introduction}

Modern generative samplers often define a continuous-time path and then approximate it with a small number of discrete evaluations. Diffusion models, flow matching, stochastic interpolants, and Schr\"odinger bridges all share this numerical bottleneck \citep{ho2020ddpm,kingma2021vdm,karras2022edm,lipman2023flow,albergo2023building,debortoli2021dsb,liu2023i2sb,tong2024simulation}. We use NFE\footnote{NFE counts neural network calls during sampling. It is a useful proxy for latency, energy use, and serving cost because these calls dominate inference in many diffusion, flow, and bridge samplers.} to denote the number of neural network function evaluations performed during sampling. At high NFE, many grids become similar because integration error is small. In the low-NFE regime, a misplaced step can dominate the behavior of the sampler. This issue is no longer only about image synthesis. The same sampling budget appears in protein design, molecular generation, materials discovery, biological trajectory modeling, and other AI-for-science systems where each extra model call can be expensive.


The standard response is to choose a useful time grid. Linear, cosine, sigmoid, power, and log-SNR grids encode different beliefs about where denoising or transport is difficult \citep{nichol2021improved,kingma2021vdm,karras2022edm,lin2024common}. Higher-order solvers and predictor-corrector schemes change the update rule \citep{lu2022dpm,lu2022dpmsolverpp,zhao2023unipc,zhang2022deis,tan2026stork}. Optimized or learned schedulers search for better timesteps or trajectory parameterizations, including Align Your Steps, optimized-step rules, dynamic-transport schedules, and recent few-step scheduler learning \citep{sabour2024alignsteps,xue2024optimized,tsimpos2025lipschitz,min2026bezierflow}. Align Your Steps optimizes a discrete schedule for diffusion sampling, Lipschitz or transport-based rules use smoothness proxies, and learned scheduler families fit parameterized curves. These methods show that the grid matters, albeit they do not by themselves say which bridge quantity should be measured. Scheduling and distillation also solve different problems: distillation changes the model weights, while the present method changes the inference grid and can therefore be combined with a distilled sampler. The distinction matters because the bridge and flow literature has moved beyond one-endpoint image diffusion. Schr\"odinger bridges now appear in simulation-free flow and score matching, image translation, energy-based sampling, biological trajectories, discrete spaces, and structured domains \citep{tong2024simulation,liu2023i2sb,liu2025asbs,tamogashev2026dataenergy,ksenofontov2025categorical,tang2026branched,wyrwal2026topological}. Flow matching has also become central in protein and molecular generation, including AlphaFlow, Proteina, BioEmu, MolFlow, FlowMM, and Wasserstein flow matching \citep{jing2024alphafoldmeetsflowmatching,proteina2025,bioemu2024,irwin2024molflow,miller2024flowmm,haviv2025wasserstein}. In these settings, a sampler is not just reversing a noising process because it transports between structured endpoints, often under geometric constraints.

Recent entropic time work gives a principled schedule for diffusion models by using conditional entropy as a time coordinate \citep{stancevic2025entropic,stancevic2026information}. A Schr\"odinger bridge carries an additional layer: conditional paths indexed by endpoint information and a marginal population flow obtained after those conditions are mixed. A single-field entropy story is therefore too small for the bridge, because the rate must compare how volume changes along the paired conditional path with how volume changes after endpoint mixing. The bridge story needs a two-term object. Let $Z$ denote the bridge condition, such as a data endpoint or endpoint pair, and let $X_t$ denote the state at time $t$. Under a smooth probability-flow description, we show that
\begin{equation}
  \frac{\dd}{\dd t}\Ent(Z\mid X_t)
  =
  \E_{Z,X_t\mid Z}\!\left[\diver \vv_t(X_t\mid Z)\right]
  -
  \E_{X_t}\!\left[\diver \bar{\vv}_t(X_t)\right].
  \label{eq:intro_identity}
\end{equation}
The first term measures conditional volume change along paired bridge paths. The second measures marginal volume change after the bridge conditions have been mixed. Their difference is the conditional--marginal (cond--marg) information-rate signal studied here as a bridge-aware scheduling signal. The proposed grid is obtained by normalizing the magnitude of this signal and inverting its cumulative distribution.
    \begin{figure}[t]
      \centering
      \includegraphics[width=\linewidth]{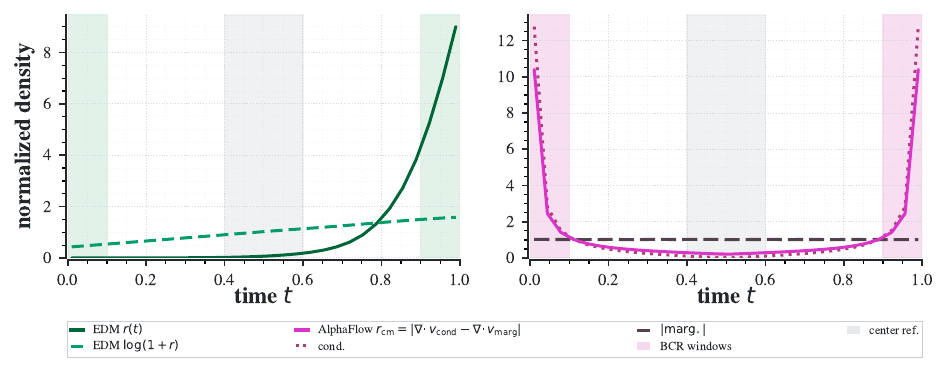}
      \caption{Entropy profiles for EDM and AlphaFlow. EDM is sharply endpoint-concentrated and requires log tempering; AlphaFlow recovers the boundary-heavy cond--marg profile predicted by the Brownian-bridge calculation. Shaded bands mark BCR endpoint windows.}
      \label{fig:entropy_curves}
    \end{figure}\textbf{}
This paper makes five contributions: (a) an entropy-rate criterion for inference-time discretization in flow and bridge samplers; (b) a cond--marg estimator that implements \Cref{eq:intro_identity}; (c) a Brownian-bridge closed form showing why the conditional term is U-shaped and singular near the endpoints; (d) a link between entropy-weighted grids and local ODE error; (e) a staged empirical path from controlled two-dimensional bridge/flow models to high-dimensional samplers.

\section{Preliminaries}
\label{sec:setup}

\textbf{Probability flows.} Let $p_0$ be a data distribution on $\R^d$ and $p_1$ be a reference distribution. A deterministic probability flow is a density path $(p_t)_{t\in[0,1]}$ and vector field $\bar{\vv}_t$ satisfying
\begin{equation}
  \partial_t p_t(\vx) + \diver\big(p_t(\vx)\bar{\vv}_t(\vx)\big)=0.
  \label{eq:continuity}
\end{equation}
Sampling integrates the learned field from the reference endpoint toward data on a decreasing time grid $1=t_0>t_1>\cdots>t_N=0$. This numerical grid is the object we study.

\textbf{Flow matching.} Flow matching often starts from a conditional interpolation. Given a condition $Z$, define a conditional law $p_t(\vx\mid Z)$ and a conditional vector field $\vv_t(\vx\mid Z)$ satisfying
\begin{equation}
  \partial_t p_t(\vx\mid Z) + \diver\big(p_t(\vx\mid Z)\vv_t(\vx\mid Z)\big)=0.
  \label{eq:conditional_continuity}
\end{equation}
The marginal field is the posterior average
  $\bar{\vv}_t(\vx)
  =
  \E\!\left[\vv_t(\vx\mid Z)\mid X_t=\vx\right]$,
whenever the conditional and marginal descriptions are compatible. This is the population field targeted by standard flow matching losses \citep{lipman2023flow,liu2023rectified,albergo2023building}.

\textbf{Optimal transport and Schr\"odinger bridges.} Optimal transport asks for a low-cost coupling between two endpoint distributions. A Schr\"odinger bridge adds stochasticity by selecting the path measure closest in relative entropy to a reference process while matching the same endpoints \citep{schrodinger1931,schrodinger1932,leonard2014schrodinger}. In the Brownian reference case, the bridge path measure can be written as a mixture of Brownian bridges weighted by an entropic optimal transport plan:
\begin{equation}
  P((X_t)_{t\in[0,1]})
  =
  \int Q((X_t)_{t\in[0,1]}\mid x_0,x_1)
  \,\dd \Pi_{2\sigma^2}(x_0,x_1).
  \label{eq:sb_mixture}
\end{equation}
For this case the natural condition is $Z=(X_0,X_1)$. The distinction between the conditional field and the marginal field is not cosmetic. It is the mathematical expression of the bridge constraint.

\textbf{Time grids and NFE.} Given a nonnegative rate $r(t)$, an entropic grid is constructed through
\begin{equation}
  Q(t)=\frac{\int_0^t r(s)\,\dd s}{\int_0^1 r(s)\,\dd s},
  \qquad
  t_k=Q^{-1}\!\left(\frac{k}{N}\right).
  \label{eq:inverse_cdf}
\end{equation}
The number of intervals, or equivalently the number of model evaluations for a one-step solver, is the NFE budget. The central question is what rate should be used for a bridge. The rate derived below is $r(t)=|\frac{\dd}{\dd t}\Ent(Z\mid X_t)|$, optionally transformed by $\log(1+r)$ when the raw signal is too concentrated near singular endpoints. Expanded derivations for the entropy identities, Brownian bridge field, and ODE allocation rule are given in \Cref{app:full_derivations,app:bb_details,app:sde_ode_details,app:ode_error}.

The empirical sections are guided by three questions summarized in \Cref{app:hypotheses}. First, does the Brownian-bridge calculation predict a U-shaped conditional rate, and do small learned two-dimensional bridge probes show the same boundary-heavy pattern? Second, does the cond--marginal estimator reveal bridge-specific allocation in a high-dimensional model? Third, when the rate is converted into a grid, does it help low-NFE sampling without being confused with endpoint quality or model training effects?



\section{Conditional--marginal entropy rate}
\label{sec:theory}

The bridge constraint asks for more than the entropy of the current marginal state. A conditional bridge path can contract because it is resolving a particular endpoint pair, while the population flow can expand or contract differently after those endpoint pairs are averaged. A schedule based on only one of these effects confounds path geometry with population geometry. The derivation below isolates the difference. It uses the standard regularity needed to differentiate entropy through a continuity equation.

\begin{assumption}
\label{ass:smooth}
All conditional and marginal densities are smooth and positive on their support, the relevant vector fields are continuously differentiable in $\vx$, and boundary terms vanish under integration by parts. The conditional and marginal paths satisfy \Cref{eq:continuity,eq:conditional_continuity}.
\end{assumption}



\begin{theorem}[Conditional--marginal entropy identity]
    \label{thm:cm_identity}
    Under \Cref{ass:smooth}, \(\displaystyle \frac{\dd}{\dd t}\Ent(Z\mid X_t)=\E_{Z,X_t\mid Z}\!\left[\diver \vv_t(X_t\mid Z)\right]-\E_{X_t}\!\left[\diver \bar{\vv}_t(X_t)\right]\). \inlineeqtag{eq:cm_identity}
    \end{theorem}


Use \(\Ent(Z\mid X_t)=\Ent(X_t\mid Z)+\Ent(Z)-\Ent(X_t)\). Since \(\Ent(Z)\) is constant and the continuity equation gives \(\frac{\dd}{\dd t}\Ent(X_t)=\E_{X_t}[\diver\bar{\vv}_t(X_t)]\), with the same calculation conditionally giving \(\frac{\dd}{\dd t}\Ent(X_t\mid Z)=\E_{Z,X_t\mid Z}[\diver\vv_t(X_t\mid Z)]\), substitution yields \Cref{eq:cm_identity}. Full integration-by-parts details are in \Cref{app:full_derivations}.

The theorem is useful because it separates endpoint-specific volume change from population volume change. In a bridge, the conditional field may expand or contract because it is resolving endpoint information, while the marginal field may show a different expansion after averaging over endpoints. The conditional--marginal rate is the difference.


\begin{proposition}[Score form of the same rate]
\label{prop:score_form}
Under the assumptions of \Cref{thm:cm_identity}, \\ \(\displaystyle \frac{\dd}{\dd t}\Ent(Z\mid X_t)=-\E_{Z,X_t\mid Z}[(\nabla_{\vx}\log p_t(X_t\mid Z)-\nabla_{\vx}\log p_t(X_t))^\top\vv_t(X_t\mid Z)]\). \inlineeqtag{eq:score_form}
\end{proposition}

\Cref{prop:score_form} is not the estimator used in the high-dimensional experiments, since conditional scores are usually not exposed by pretrained models. It is included because it explains what the divergence contrast measures: the conditional field is weighted by the gap between the conditional score and the marginal score.

\subsection{Gaussian bridges}
\label{sec:gaussian}

For a Gaussian conditional path,
$
  p_t(\vx\mid z)=\mathcal{N}\big(\vx;\mu(z,t),\sigma^2(z,t)I\big),
$
a sample can be written as $X_t=\mu(z,t)+\sigma(z,t)\veps$, with $\veps\sim \mathcal{N}(0,I)$. Differentiating the sample path gives
$
  \vv_t(\vx\mid z)
  =
  \partial_t \mu(z,t)
  +
  \partial_t \sigma(z,t)\frac{\vx-\mu(z,t)}{\sigma(z,t)}.
$
Using the Gaussian score
$
  \nabla_{\vx}\log p_t(\vx\mid z)
  =
  -\frac{\vx-\mu(z,t)}{\sigma^2(z,t)},
$
we obtain the score form
\begin{equation}
  \vv_t(\vx\mid z)
  =
  \partial_t\mu(z,t)
  -
  \sigma(z,t)\partial_t\sigma(z,t)\nabla_{\vx}\log p_t(\vx\mid z).
  \label{eq:gaussian_velocity_score}
\end{equation}

For a Brownian bridge between $x_0$ and $x_1$, write
  $m_t=(1-t)x_0+tx_1$ 
  and
  $\sigma(t)=\sigma_0\sqrt{t(1-t)}$.
Substitution into \Cref{eq:gaussian_velocity_score} gives the deterministic probability-flow field
\begin{equation}
  \vv_t(\vx\mid x_0,x_1)
  =
  (x_1-x_0)
  +
  \frac{1-2t}{2t(1-t)}(\vx-m_t).
  \label{eq:bb_ode_field}
\end{equation}
The noise scale $\sigma_0$ cancels. This is important because the schedule shape is a property of bridge geometry, not of an arbitrary Brownian scale.

\begin{proposition}[Closed-form conditional divergence]
    \label{prop:bb_divergence}
    For the Brownian-bridge probability-flow field in \Cref{eq:bb_ode_field}, \(\displaystyle \diver \vv_t(\vx\mid x_0,x_1)=d\,\frac{1-2t}{2t(1-t)}\). \inlineeqtag{eq:bb_divergence}
    \end{proposition}

\begin{proof}
The term $x_1-x_0$ is independent of $\vx$ and has zero divergence. The term $m_t$ is also independent of $\vx$. Therefore
\begin{equation}
  \diver \left[\frac{1-2t}{2t(1-t)}(\vx-m_t)\right]
  =
  \frac{1-2t}{2t(1-t)}\diver \vx
  =
  d\,\frac{1-2t}{2t(1-t)}.
\end{equation}
\end{proof}

The magnitude of \Cref{eq:bb_divergence} is singular near both endpoints and vanishes at the midpoint. This is the mathematical source of the U-shaped bridge profile. Because the conditional divergence is spatially constant and endpoint-independent in the Brownian case, averaging over an entropic optimal-transport coupling preserves the same conditional expectation. The full conditional--marginal rate further subtracts the marginal divergence term in \Cref{eq:cm_identity}; the conditional closed form supplies a model-independent anchor and the marginal term adapts it to the learned population path. The U-shape is a theorem for the Gaussian Brownian bridge, not for every learned bridge. Outside this class, the identity remains the general object and the profile must be estimated or derived from the model. Other interpolants can yield flat or monotone analytic profiles, as summarized in \Cref{app:interpolation_profiles}.

\subsection{SDE drift and probability flow}
\label{sec:sde_ode}

For a Brownian bridge, the conditional SDE drift can be written
\begin{equation}
  u^o_t(\vx\mid x_0,x_1)
  =
  (x_1-x_0)
  +
  \frac{1-2t}{t(1-t)}(\vx-m_t),
  \label{eq:bb_sde_drift}
\end{equation}
with conditional score
\begin{equation}
  \nabla_{\vx}\log p_t(\vx\mid x_0,x_1)
  =
  \frac{m_t-\vx}{\sigma_0^2t(1-t)}.
  \label{eq:bb_cond_score}
\end{equation}
The ODE field in \Cref{eq:bb_ode_field} and the SDE drift in \Cref{eq:bb_sde_drift} are distinct conditional objects. They satisfy
\begin{equation}
  u^o_t(\vx\mid x_0,x_1)
  =
  2\vv_t(\vx\mid x_0,x_1)-(x_1-x_0).
  \label{eq:sde_ode_cond_relation}
\end{equation}
The usual probability-flow relation applies to marginal fields after integrating over bridge endpoints. It should not be applied naively to the conditional drift. \Cref{app:sde_ode_details} gives the step-by-step algebra, including the cancellation that explains the factor of two. This distinction is operationally important: using the conditional SDE drift as if it were already the probability-flow field removes the bridge contraction term and makes the conditional profile appear artificially flat.

\section{Estimator, grids, and experimental loops}
\label{sec:estimator}

We organize the method into three layers. The inner layer is mathematical: choose the rate that a bridge actually exposes, namely the conditional divergence minus the marginal divergence. The middle layer is computational: estimate that rate without forming a Jacobian. The outer layer is empirical: use the estimated rate to build a grid, then ask whether the resulting allocation explains the observed behavior of high-dimensional samplers. The layers are tied together because the outer evidence is meaningful only if the computational estimator matches the mathematical rate.

The cond--marg rate is estimated by evaluating both divergence terms in \Cref{eq:cm_identity}. Exact traces are infeasible in high dimension, so the middle loop uses Hutchinson estimation:
\begin{equation}
  \diver \vv_t(\vx)
  =
  \Tr \nabla_{\vx}\vv_t(\vx)
  =
  \E_{\vu}\!\left[\vu^\top\nabla_{\vx}\vv_t(\vx)\vu\right],
  \qquad
  \E[\vu\vu^\top]=I.
  \label{eq:hutch}
\end{equation}
Using the same random probe for the conditional and marginal terms reduces the variance of their difference. After averaging over calibration samples and time points, the rate is smoothed, clipped to a small positive floor, transformed by the default regularized map $\phi(r)=\log(1+r)$, and converted into a grid through \Cref{eq:inverse_cdf}. We also report the raw transform $\phi(r)=r$  as an ablation. The raw grid preserves the uncompressed cond--marg signal, whereas the log transform keeps the ordering of high-rate regions while reducing endpoint domination.


\begin{computebox}
\textbf{Computational note.} With $M$ calibration times, $n$ calibration states, and $m$ Hutchinson probes, the estimator uses $O(Mnm)$ derivative-vector products. A typical calibration with 50 times and four probes uses 200 derivative-vector products, which is then amortized across all samples drawn with the same grid. Once the grid is built, sampling at a fixed NFE has the same number of model evaluations as linear, cosine, sigmoid, or power grids. The high-dimensional evidence uses the same inference budgets as the baselines and treats entropy estimation as a calibration step. Memory and arithmetic-intensity details are in \Cref{app:compute}.
\end{computebox}

The output of this construction is a density over time and the grid obtained from its inverse CDF. \Cref{fig:alphaflow_schedule} summarizes the AlphaFlow cond--marg allocation, while \Cref{fig:alphaflow_schedule_geometry} compares the regularized node geometry against linear and against the ten-step DTW alignment.
    \begin{figure}[t]
      \centering
      \includegraphics[height=0.15\textheight, width=0.95\linewidth]{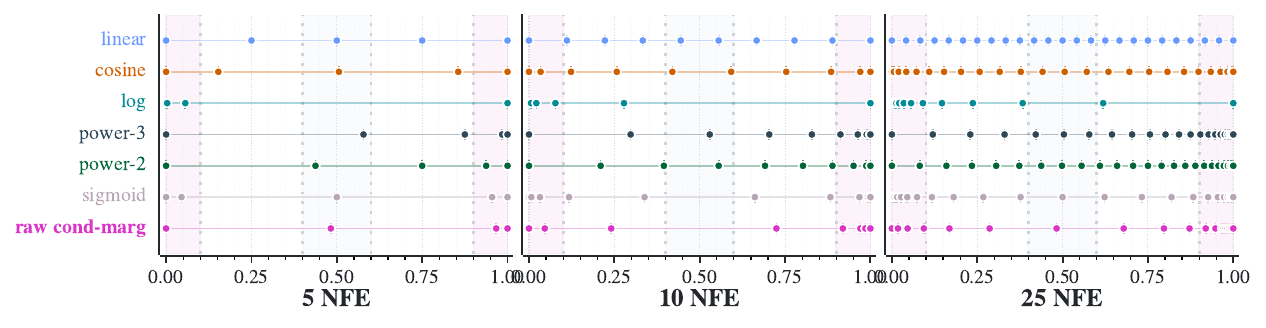}
      \caption{AlphaFlow cond--marg schedule allocation. The inverse-CDF grid places nodes according to the estimated conditional--marginal rate for 5, 10, and 25 steps; shaded endpoint bands mark the boundary regions used for the schedule-geometry diagnostic. This is grid-allocation evidence, not endpoint-quality evidence.}
      \label{fig:alphaflow_schedule}
    \end{figure}

Algorithmic summaries are given in \Cref{app:algorithms}. We also define a boundary concentration ratio (BCR) to summarize how much grid mass is placed near the endpoints. Visual inspection of schedule curves can be misleading when two grids look similar but allocate different endpoint density, so BCR gives a quantitative check of boundary concentration. This auxiliary measurement can be useful beyond this work for comparing schedule concentration across models; details and derivation are in \Cref{app:bcr}. We use BCR only as a grid-geometry diagnostic, not as an endpoint-quality metric.

\begin{proposition}[Weighted step-size rule]
    \label{prop:ode_error}
    For an order-$p$ one-step ODE solver with local error constant $C(t)$, the continuous step density minimizing $\int_0^1 C(t)\rho(t)^{-p}\dd t$ subject to $\int_0^1\rho(t)\dd t=1$ is \(\displaystyle \rho^\star(t)\propto C(t)^{1/(p+1)}\). \inlineeqtag{eq:rho_star}
    \end{proposition}

The entropy rate is not asserted to equal $C(t)$ exactly. It is used as a tractable proxy for regions where the vector field compresses, expands, or bends sharply. Therefore, the default entropic grid uses the regularized rate $\log(1+r)$. The raw rate is kept as an ablation because it preserves the mathematical signal but can over-concentrate the grid near singular endpoints. This is the same practical lesson seen across diffusion scheduler and fast-solver work: a principled time signal still needs a numerically stable parametrization at low-NFE \citep{karras2022edm, lin2024common, chen2023importance, lu2022dpm, stancevic2025entropic}. The same rate can matter differently for deterministic and stochastic solvers: ODE integration exposes grid-placement error directly, while SDE noise can smooth some placement errors. The full derivation of \Cref{prop:ode_error} is in \Cref{app:ode_error}.

\section{Results}
\label{sec:evidence}

Before moving to high-dimensional samplers, we ran a controlled two-dimensional pre-flight experiment. We trained residual vector-field models on four transport geometries: continuous--continuous (C-C), continuous--discrete (C-D), discrete--continuous (D-C), and discrete--discrete (D-D), across bridge scales. We then built grids from the measured entropy-rate profiles. These learned probes recover the boundary-heavy Brownian-bridge pattern, place more nodes near endpoints, and show the largest gains at low NFE before schedules converge. We report both ODE and SDE sweeps, but use the probability-flow ODE for the high-dimensional study because it removes sampling noise and matches the cond--marg estimator. Full transport visuals and ODE/SDE curves are in \Cref{app:figures} and the controlled transport setup and entropy maps are summarized in in  \Cref{fig:controlled_transport_setup_appendix,fig:mmd_steps_appendix}.


\begin{table}[!t]
  \vspace{-0.75em}
  \centering
  \small
  \setlength{\tabcolsep}{3pt}
  \caption{Synthetic low-NFE results. Left: learned ODE-Heun MMD averaged across steps, scenarios, $\sigma$, and seeds. Right: entropic-vs-linear improvement with bootstrap intervals.}
  \label{tab:low_nfe_improvement}

  \begin{minipage}[t]{0.49\linewidth}
    \centering
    \scriptsize
    \renewcommand{\arraystretch}{1.05}
    \begin{tabular*}{\linewidth}{@{\extracolsep{\fill}}lcc@{}}
      \toprule
      Scheduler & MMD $\downarrow$ ($\times 10^{-3}$) & Entropic vs (\%) \\
      \midrule
      Linear  & 6.028$\pm$0.603 & 8.1 [6.6, 9.7] \\
      Power-2 & 6.620$\pm$1.070 & 16.4 [14.1, 18.6] \\
      Power-3 & 6.912$\pm$1.216 & 19.9 [17.6, 22.3] \\
      Log     & 8.164$\pm$1.855 & 32.2 [29.7, 34.6] \\
      Cosine  & 5.548$\pm$0.152 & 0.2 [0.1, 0.3] \\
      Sigmoid & 5.639$\pm$0.339 & 1.8 [0.5, 3.1] \\
      \begingroup\setlength{\fboxsep}{1.5pt}\colorbox{yellow!30}{\textbf{Entropic}}\endgroup
      &
      \begingroup\setlength{\fboxsep}{1.5pt}\colorbox{yellow!30}{\textbf{5.537$\pm$0.131}}\endgroup
      &
      \begingroup\setlength{\fboxsep}{1.5pt}\colorbox{yellow!30}{\textbf{---}}\endgroup \\
      \bottomrule
    \end{tabular*}
  \end{minipage}
  \hfill
  \begin{minipage}[t]{0.47\linewidth}
    \centering
    \scriptsize
    \renewcommand{\arraystretch}{1.05}
    \begin{tabular*}{\linewidth}{@{\extracolsep{\fill}}ccc@{}}
      \toprule
      Steps & ODE-Heun (\%) & SDE-Heun (\%) \\
      \midrule
      10 & 18.1 [-4.9, 35.9]  & 22.7 [7.3, 34.9] \\
      25 & 12.3 [-10.8, 31.1] & 13.1 [-3.7, 27.0] \\
      50 & 7.0 [-16.8, 26.7]  & 10.7 [-5.2, 24.8] \\
      \bottomrule
    \end{tabular*}
  \end{minipage}
  \vspace{-0.75em}
\end{table}
We next test the same allocation signal in two high-dimensional settings. EDM/CIFAR-10 is a mismatch case: it is not a bridge model and was not trained with the entropy objective, but it tests whether a tempered cond--marg-inspired grid can help a pretrained image sampler \citep{karras2022edm}. AlphaFlow operates on protein ensembles and is closer to the bridge setting because its flow-matching sampler exposes the fields needed for the cond--marg estimator, though it was also not trained with this criterion \citep{jing2024alphafoldmeetsflowmatching}. We use ODE sampling throughout, evaluate CAMEO22 and ATLAS proteins\citep{ha2015cameo, stam2024atlas}, and report the evaluation funnel in \Cref{app:plddt_metric}.



\textbf{AlphaFlow.} The cond--marg profile in \Cref{fig:entropy_curves} is strongly U-shaped: the intervals $[0,0.1]$ and $[0.9,1]$ each carry about $31.1\%$ of the normalized mass, while the center window $[0.4,0.6]$ carries about $5.0\%$.    The resulting grid allocation is shown in \Cref{fig:alphaflow_schedule}; the log1p node geometry and ten-step DTW alignment are shown in \Cref{fig:alphaflow_schedule_geometry}. Protein structures lie on a thin constrained manifold, where bond geometry, steric exclusions, chirality, and fold commitment can make both endpoints stiff. This motivates boundary allocation, but it does not by itself prove endpoint-quality improvement. Endpoint predicted local distance difference test (pLDDT) is mixed in the available AlphaFlow sweeps summarized in \Cref{tab:alphaflow-plddt-ode}: sigmoid and cosine are strongest in many displayed cells, raw entropic scheduling is strongest in the medium and large five-step cells, and raw entropic pLDDT degrades sharply at 10 and 25 steps. pLDDT is an AlphaFold-style per-residue local distance difference test, and we use it here as an endpoint confidence proxy rather than as a measure of schedule geometry \citep{jumper2021,mariani2013lddt}. Our method leads pLDDT performance on both medium and large-sized proteins ($97.62\pm0.22$ and $97.91\pm0.14$) at the low 5-step NFE count, and it is close to the sigmoid leader for small proteins ($95.06\pm1.12$ versus $95.26\pm0.99$). Its limitations in this paper are summarized in \Cref{app:plddt_metric}.  Additional AlphaFlow grid and structure-evolution visualizations are kept in the Evaluation Grids appendix, including the 25-step AlphaFlow trajectory diagnostic in \Cref{app:evaluation_grids}; \Cref{fig:alphaflow_evolution_25steps_appendix} explains how endpoint plDDT should be interpreted so that we keep our emphasis on the schedule geometry.

\begin{table}[H]
\centering
\scriptsize
\setlength{\tabcolsep}{2.2pt}
\caption{AlphaFlow endpoint pLDDT under ODE-Heun. Values are means with 95\% CIs; higher is better. Highlighting marks the best scheduler per column and bold marks entropic.}
\label{tab:alphaflow-plddt-ode}
\begin{tabular*}{\linewidth}{@{\extracolsep{\fill}}lccccccccc@{}}
    \toprule
    & \multicolumn{3}{c}{\begin{tabular}[c]{@{}c@{}}Small \,\,{\scriptsize $\le 50$ aa}\end{tabular}}
    & \multicolumn{3}{c}{\begin{tabular}[c]{@{}c@{}}Medium \,\, {\scriptsize 51--400 aa}\end{tabular}}
    & \multicolumn{3}{c}{\begin{tabular}[c]{@{}c@{}}Large \,\,{\scriptsize $>400$ aa}\end{tabular}} \\
    \cmidrule(lr){2-4}\cmidrule(lr){5-7}\cmidrule(l){8-10}
    Scheduler & 5 & 10 & 25 & 5 & 10 & 25 & 5 & 10 & 25 \\
    \midrule
    Linear  & \alphaplddtcell{89.97}{1.83} & \alphaplddtcell{94.15}{0.88} & \alphaplddtcell{95.82}{0.51} & \alphaplddtcell{93.10}{0.61} & \alphaplddtcell{95.08}{0.31} & \alphaplddtcell{96.25}{0.12} & \alphaplddtcell{94.27}{0.57} & \alphaplddtcell{94.30}{0.30} & \alphaplddtcell{94.12}{0.12} \\
    Cosine  & \alphaplddtcell{92.73}{1.44} & \alphaplddtcell{96.08}{0.66} & \alphabestplddtcell{96.07}{0.54} & \alphaplddtcell{95.23}{0.47} & \alphaplddtcell{97.57}{0.13} & \alphaplddtcell{97.54}{0.06} & \alphaplddtcell{95.92}{0.45} & \alphaplddtcell{97.09}{0.13} & \alphaplddtcell{96.39}{0.07} \\
    Sigmoid & \alphabestplddtcell{95.26}{0.99} & \alphabestplddtcell{96.15}{0.80} & \alphaplddtcell{95.64}{0.71} & \alphaplddtcell{97.49}{0.26} & \alphabestplddtcell{98.11}{0.08} & \alphabestplddtcell{97.55}{0.06} & \alphaplddtcell{97.82}{0.23} & \alphabestplddtcell{97.95}{0.08} & \alphabestplddtcell{96.67}{0.06} \\
    Power-2 & \alphaplddtcell{84.81}{2.64} & \alphaplddtcell{90.44}{1.82} & \alphaplddtcell{94.87}{0.78} & \alphaplddtcell{90.92}{0.80} & \alphaplddtcell{93.99}{0.52} & \alphaplddtcell{95.27}{0.22} & \alphaplddtcell{92.54}{0.73} & \alphaplddtcell{94.89}{0.45} & \alphaplddtcell{93.52}{0.20} \\
    \textbf{Entropic} & \alphaplddtcellbold{95.06}{1.12} & \alphaplddtcellbold{78.79}{3.42} & \alphaplddtcellbold{75.45}{3.81} & \alphabestplddtcellbold{97.62}{0.22} & \alphaplddtcellbold{87.85}{1.11} & \alphaplddtcellbold{86.02}{1.25} & \alphabestplddtcellbold{97.91}{0.14} & \alphaplddtcellbold{90.26}{0.68} & \alphaplddtcellbold{89.26}{0.76} \\
    \bottomrule
\end{tabular*}
\end{table}

\textbf{EDM.} EDM is a mismatch test because it lies outside the Schr\"odinger-bridge setting \citep{karras2022edm}. We use log-tempered cond--marg entropy as the default scheduler and raw entropy as the ablation: raw allocation over-concentrates endpoints and performs poorly, while \(\log(1+r)\) keeps the high-rate regions without exhausting the low-NFE budget at singular endpoints. \Cref{tab:edm-fid-ode} reports the full ODE-Heun FID sweep: at five steps entropic achieves $186.26\pm3.97$, compared with $200.52\pm2.91$ for linear, $238.03\pm5.29$ for cosine, and $355.06\pm2.20$ for sigmoid, while 10/25-step schedules are close.     Additional trajectory-grid diagnostics are in \Cref{app:evaluation_grids}. The schedule-geometry comparison in \Cref{fig:alphaflow_schedule_geometry} shows the same calibration point: log1p stays close to linear but remains endpoint-biased, whereas raw cond--marg allocation requires stronger warping.

\begin{table}[!t]
  \caption{EEDM/CIFAR-10 ODE-Heun FID. Values are means with 95\% CIs over five seeds; lower is better. Entropic uses cond--marg log1p, with raw entropy shown as an ablation.}
  \label{tab:edm-fid-ode}
  \centering
  \scriptsize
  \setlength{\tabcolsep}{1.8pt}
  \begin{adjustbox}{max width=\linewidth}
  \begin{tabular}{@{}lccc@{}}
    \toprule
    Scheduler & 5 steps & 10 steps & 25 steps \\
    \midrule
    ODE + Linear  & \edmfidcell{200.52}{2.91} & \edmfidcell{172.78}{2.83} & \edmfidcell{170.85}{3.50} \\
    ODE + Cosine  & \edmfidcell{238.03}{5.29} & \edmfidcell{176.58}{2.49} & \edmfidcell{170.80}{3.36} \\
    ODE + Sigmoid & \edmfidcell{355.06}{2.20} & \edmfidcell{212.15}{8.03} & \edmfidcell{172.68}{2.39} \\
    ODE + Power-2 & \edmfidcell{316.54}{2.68} & \edmfidcell{189.13}{3.51} & \edmfidcell{171.28}{4.06} \\
    ODE + Entropic (raw) & \edmfidcell{322.58}{4.80} & \edmfidcell{290.11}{4.94} & \edmfidcell{227.96}{1.43} \\
    \textbf{ODE + Entropic} & \edmbestfidcell{186.26}{3.97} & \edmbestfidcell{1å72.10}{3.74} & \edmbestfidcell{170.78}{3.42} \\
    \bottomrule
  \end{tabular}
  \end{adjustbox}
\end{table}

  \begin{figure*}[t]
    \centering
    \includegraphics[width=\textwidth]{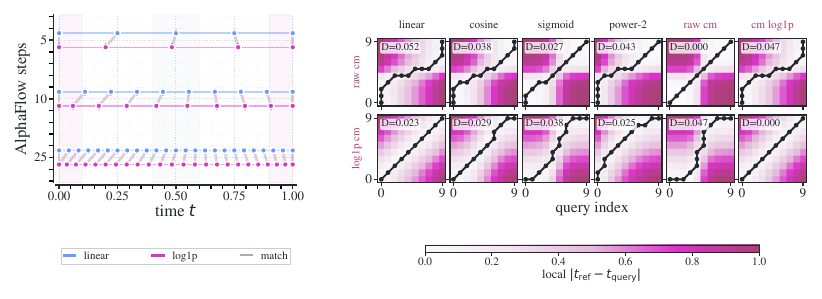}
    \caption{AlphaFlow schedule-geometry diagnostics. Left: the conditional--marginal log1p grid remains \textit{visually} close to the linear grid across 5, 10, and 25 steps while keeping
  endpoint bias quantitatively. Right: dynamic-time-warping paths at 10 steps quantify this distinction; the log1p grid has lower distance to linear than the raw cond--marg grid, while raw cond--marg requires stronger index warping.}
    \label{fig:alphaflow_schedule_geometry}
  \end{figure*}

    \textbf{Cross-domain allocation.} \Cref{fig:entropy_curves} puts the EDM and AlphaFlow entropy profiles on the same page, \Cref{fig:alphaflow_schedule} shows the AlphaFlow allocation, and \Cref{fig:alphaflow_schedule_geometry} shows the log1p-vs-linear node geometry together with the ten-step DTW alignment. In these diagnostics the bridge-aware quantity is the two-term rate
\begin{equation}
  r_{\cm}(t)
  =
  \left|
  \E_{Z,X_t\mid Z}\!\left[\diver\vv_t(X_t\mid Z)\right]
  -
  \E_{X_t}\!\left[\diver\bar{\vv}_t(X_t)\right]
  \right|.
  \label{eq:results_cm_rate}
\end{equation}
The profiles are not the same, which is the point of measuring the entropy rate. EDM needs tempering because its raw profile is too sharp for stable low-NFE integration. AlphaFlow shows a bridge-like U-shape, but endpoint confidence remains influenced by the training objective and protein-specific geometry. The evidence separates three ideas that are easy to conflate: the entropy profile, the numerical grid derived from that profile, and the downstream metric used to judge the generated sample. Additional grid and evolution visualizations for EDM and AlphaFlow are in \Cref{app:evaluation_grids}.

\section{Discussion}
\label{sec:discussion}

This work follows a simple sequence of tests. The first obstacle is remaining competitive with, and sometimes improving over, the strong fixed grid baselines. Linear, cosine, sigmoid, and power schedules often work because they encode plausible assumptions about where a trajectory is hard to integrate. The entropy construction must do more than replace one smooth curve with another, because it must measure a path-dependent signal that fixed grids can only approximate. Cosine, sigmoid, and power grids can be good because they approximate common shapes of entropy-rate profiles. The two-dimensional sweep makes this point concrete,although it does not prove that any fixed heuristic is optimal. In those controlled experiments, cosine was often close to the entropic scheduler because its node density can approximate a boundary-heavy curve. However, this was not a separate theory of cosine schedules, and the near-tie was scenario-dependent. The result should be read as shape compatibility between cosine and the measured entropic grid, not as equivalence between the two schedules \citep{nichol2021improved,karras2022edm,lin2024common}. When the true profile is monotone, U-shaped, or boundary-heavy, these heuristics can be close. Linear grids can work when the profile is flat. The advantage of the entropy construction is that it measures the shape instead of assuming it. The same 2D runs also explain why the training-loss curves matter. In the flow-matching ideal, reducing the loss should make the learned vector field closer to the target bridge field on the training distribution. When train and validation curves decrease together, the estimated entropy-rate profile is expected to move toward the analytic U-shape. We use this only as a diagnostic for alignment between the learned field and the analytic bridge calculation, not as a convergence proof for all learned bridges. This distinction matters most at low NFE, where a few misplaced steps can determine the outcome.

\section{Related work}
\label{sec:related}

\textbf{Diffusion, flow, and bridge models.} Diffusion models and score-based samplers established continuous-time generation as a numerical integration problem \citep{ho2020ddpm,kingma2021vdm,karras2022edm}. Flow matching and stochastic interpolants later simplified training by regressing vector fields along prescribed conditional paths \citep{lipman2023flow,albergo2023building,tong2024cfm,liu2023rectified}. Schr\"odinger bridges connect this view to entropy-regularized optimal transport and stochastic control \citep{schrodinger1931,schrodinger1932,leonard2014schrodinger,debortoli2021dsb,liu2023i2sb,tong2024simulation}. Recent work expands bridges to energy-based sampling, discrete spaces, structured domains, and branched trajectories \citep{liu2025asbs,tamogashev2026dataenergy,ksenofontov2025categorical,tang2026branched,wyrwal2026topological}. Our contribution is orthogonal to these training objectives: it asks how to place inference steps once a path has been learned.

\textbf{Schedulers and fast samplers.} Classical schedules and design-space analyses remain important baselines \citep{nichol2021improved,kingma2021vdm,karras2022edm,lin2024common}. DPM-Solver, DPM-Solver++, UniPC, DEIS, and STORK change the numerical update or exploit ODE structure \citep{lu2022dpm,lu2022dpmsolverpp,zhao2023unipc,zhang2022deis,tan2026stork}. Align Your Steps, optimized-step methods, dynamic-transport scheduling, and B{\'e}zierFlow search for better timesteps or scheduler parameterizations \citep{sabour2024alignsteps,xue2024optimized,tsimpos2025lipschitz,min2026bezierflow}. These approaches can be combined with an entropy-derived grid, but they do not identify the bridge-specific conditional--marginal rate.



\textbf{Entropy and information geometry.} Entropy-based scheduling for diffusion models was introduced by \citet{stancevic2025entropic} and further connected to information dynamics by \citet{stancevic2026information}. Entropy and mutual information also appear widely as learning objectives or regularizers: entropy-regularized optimal transport and maximum-entropy control use entropy to shape couplings or policies \citep{cuturi2013sinkhorn,ziebart2008maximum,haarnoja2018soft}; information bottlenecks, InfoMax, InfoGAN, neural mutual-information estimators, and Mutual Information Machines use mutual information as a compression, representation, or generative-model objective \citep{tishby2000information,bell1995infomax,chen2016infogan,belghazi2018mine,livne2019mim}; and protein inverse folding has used diversity-regularized DPO to increase differential entropy in ProteinMPNN log-probability space \citep{park2024peptidedpo}. This work extends the information view to bridge and flow settings through \Cref{eq:cm_identity}. The main difference is the object being estimated: diffusion entropic time uses a one-endpoint entropy signal, while a bridge exposes a contrast between conditional and marginal divergence; \Cref{tab:stancevic_appendix} summarizes the comparison.

\textbf{Low-step training and scientific domains.} Distillation compresses many steps into a trained few-step model \citep{gushchin2025ibmd,he2024cdbm,ai2026jointdistillation,passaro2026stochasticfewstep}. Entropic discretization is training-free and can be used with pretrained samplers. In scientific domains, flow and bridge methods increasingly model constrained objects such as proteins, molecules, materials, and distributions over distributions \citep{jing2024alphafoldmeetsflowmatching,proteina2025,bioemu2024,irwin2024molflow,miller2024flowmm,haviv2025wasserstein}. These domains motivate measuring the path geometry directly because heuristic grids may miss endpoint stiffness or manifold constraints. A compact comparison is given in \Cref{app:comparison}.
\section{Limitations and future work}
\label{sec:limitations}


Overall, the results support conditional--marginal entropy rate as a practical low-NFE allocation signal, provided it is regularized into a stable grid and interpreted separately from endpoint-quality metrics. The estimator depends on access to conditional and marginal fields. Some pretrained models expose only a single field or require a surrogate for one of the terms. In such cases the resulting grid should be treated as a diagnostic unless both terms are matched to the theory. Hutchinson estimation introduces variance, and raw boundary singularities can be too sharp for numerical stability. Log tempering is a practical calibration, not a theorem. The closed-form U-shape is also a Gaussian Brownian-bridge result. For non-Gaussian learned bridges, \Cref{eq:cm_identity} remains the target, but the shape of the rate is an empirical or model-specific question. While the current results are encouraging, optimized entropic scheduling for protein generation deserves a separate research program. Short proteins, long proteins, flexible loops, and disordered regions may have different entropy profiles. Models such as AlphaFlow, BioEmu, Proteina, and Peptron make it possible to study those profiles as biological signals rather than as nuisance quantities \citep{jing2024alphafoldmeetsflowmatching,bioemu2024,proteina2025, peptron2025}. Besides protein generation, the same question could be tested in image generation by training or calibrating EDM- or latent-diffusion models with the entropy criterion, and then sampling them with an entropic grid at inference time \citep{karras2022edm,rombach2022ldm,ho2020ddpm}. In both domains, the next question is whether training and inference can be aligned around the same cond--marg information rate.



%

 \bibliographystyle{plainnat}
\bibliography{references}

@article{mariani2013lddt,
  title   = {lDDT: a local superposition-free score for comparing protein structures and models using distance difference tests},
  author  = {Mariani, Valerio and Biasini, Marco and Barbato, Alessandro and Schwede, Torsten},
  journal = {Bioinformatics},
  volume  = {29},
  number  = {21},
  pages   = {2722--2728},
  year    = {2013},
  doi     = {10.1093/bioinformatics/btt473}
}

@inproceedings{ho2020ddpm,
  title     = {Denoising Diffusion Probabilistic Models},
  author    = {Ho, Jonathan and Jain, Ajay and Abbeel, Pieter},
  booktitle = {Advances in Neural Information Processing Systems},
  volume    = {33},
  pages     = {6840--6851},
  year      = {2020}
}

@inproceedings{nichol2021improved,
  title     = {Improved Denoising Diffusion Probabilistic Models},
  author    = {Nichol, Alexander Quinn and Dhariwal, Prafulla},
  booktitle = {International Conference on Machine Learning},
  pages     = {8162--8171},
  year      = {2021}
}

@inproceedings{rombach2022ldm,
  title     = {High-Resolution Image Synthesis With Latent Diffusion Models},
  author    = {Rombach, Robin and Blattmann, Andreas and Lorenz, Dominik and Esser, Patrick and Ommer, Bj{\"o}rn},
  booktitle = {Proceedings of the IEEE/CVF Conference on Computer Vision and Pattern Recognition},
  pages     = {10684--10695},
  year      = {2022}
}

@inproceedings{kingma2021vdm,
  title     = {Variational Diffusion Models},
  author    = {Kingma, Diederik P. and Salimans, Tim and Poole, Ben and Ho, Jonathan},
  booktitle = {Advances in Neural Information Processing Systems},
  volume    = {34},
  year      = {2021}
}

@inproceedings{karras2022edm,
  title     = {Elucidating the Design Space of Diffusion-Based Generative Models},
  author    = {Karras, Tero and Aittala, Miika and Aila, Timo and Laine, Samuli},
  booktitle = {Advances in Neural Information Processing Systems},
  volume    = {35},
  pages     = {26565--26577},
  year      = {2022}
}

@inproceedings{lin2024common,
  title     = {Common Diffusion Noise Schedules and Sample Steps are Flawed},
  author    = {Lin, Shanchuan and Liu, Bingchen and Li, Jiashi and Yang, Xiao},
  booktitle = {IEEE/CVF Winter Conference on Applications of Computer Vision},
  pages     = {5404--5411},
  year      = {2024}
}

@inproceedings{lu2022dpm,
  title     = {{DPM-Solver}: A Fast {ODE} Solver for Diffusion Probabilistic Model Sampling in Around 10 Steps},
  author    = {Lu, Cheng and Zhou, Yuhao and Bao, Fan and Chen, Jianfei and Li, Chongxuan and Zhu, Jun},
  booktitle = {Advances in Neural Information Processing Systems},
  volume    = {35},
  pages     = {5775--5787},
  year      = {2022}
}

@inproceedings{zhao2023unipc,
  title     = {{UniPC}: A Unified Predictor-Corrector Framework for Fast Sampling of Diffusion Models},
  author    = {Zhao, Wenliang and Bai, Lujia and Rao, Yongming and Zhou, Jie and Lu, Jiwen},
  booktitle = {Advances in Neural Information Processing Systems},
  volume    = {36},
  year      = {2023}
}

@article{xue2024optimized,
  title   = {Accelerating Diffusion Sampling with Optimized Time Steps},
  author  = {Xue, Shuchen and Liu, Zhaoqiang and Chen, Fei and Zhang, Shifeng and Hu, Tianyang and Xing, Enze and Zhou, Mingyuan},
  journal = {arXiv preprint arXiv:2402.17376},
  year    = {2024}
}

@inproceedings{lipman2023flow,
  title     = {Flow Matching for Generative Modeling},
  author    = {Lipman, Yaron and Chen, Ricky T. Q. and Ben-Hamu, Heli and Nickel, Maximilian and Le, Matthew},
  booktitle = {International Conference on Learning Representations},
  year      = {2023}
}

@inproceedings{albergo2023building,
  title     = {Building Normalizing Flows with Stochastic Interpolants},
  author    = {Albergo, Michael S. and Vanden-Eijnden, Eric},
  booktitle = {International Conference on Learning Representations},
  year      = {2023}
}

@article{tong2024cfm,
  title   = {Improving and Generalizing Flow-Based Generative Models with Minibatch Optimal Transport},
  author  = {Tong, Alexander and Fatras, Kilian and Malkin, Nikolay and Huguet, Guillaume and Zhang, Yanlei and Rector-Brooks, Jarrid and Wolf, Guy and Bengio, Yoshua},
  journal = {Transactions on Machine Learning Research},
  year    = {2024}
}

@article{schrodinger1931,
  title   = {{\"U}ber die Umkehrung der Naturgesetze},
  author  = {Schr{\"o}dinger, Erwin},
  journal = {Sitzungsberichte der Preussischen Akademie der Wissenschaften, Physikalisch-mathematische Klasse},
  year    = {1931}
}

@article{schrodinger1932,
  title   = {{\"U}ber die Umkehrung der Naturgesetze. {Z}weite Mitteilung},
  author  = {Schr{\"o}dinger, Erwin},
  journal = {Sitzungsberichte der Preussischen Akademie der Wissenschaften, Physikalisch-mathematische Klasse},
  year    = {1932}
}

@article{leonard2014schrodinger,
  title   = {A survey of the Schr{\"o}dinger problem and some of its connections with optimal transport},
  author  = {L{\'e}onard, Christian},
  journal = {Discrete and Continuous Dynamical Systems - A},
  volume  = {34},
  number  = {4},
  pages   = {1533--1574},
  year    = {2014},
  doi     = {10.3934/dcds.2014.34.1533}
}

@inproceedings{debortoli2021dsb,
  title     = {Diffusion {S}chr{\"o}dinger Bridge with Applications to Score-Based Generative Modeling},
  author    = {De Bortoli, Valentin and Thornton, James and Heng, Jeremy and Doucet, Arnaud},
  booktitle = {Advances in Neural Information Processing Systems},
  volume    = {34},
  pages     = {17695--17709},
  year      = {2021}
}

@inproceedings{liu2023i2sb,
  title     = {{I$^2$SB}: Image-to-Image {S}chr{\"o}dinger Bridge},
  author    = {Liu, Guan-Horng and Vahdat, Arash and Huang, De-An and Theodorou, Evangelos A. and Nie, Weili and Anandkumar, Anima},
  booktitle = {International Conference on Machine Learning},
  pages     = {22042--22062},
  year      = {2023}
}

@article{tong2024simulation,
  title   = {Simulation-Free {S}chr{\"o}dinger Bridges via Score and Flow Matching},
  author  = {Tong, Alexander and Malkin, Nikolay and Huguet, Guillaume and Zhang, Yanlei and Rector-Brooks, Jarrid and Fatras, Kilian and Wolf, Guy and Bengio, Yoshua},
  journal = {arXiv preprint arXiv:2307.03672},
  year    = {2024}
}

@inproceedings{stancevic2025entropic,
  title     = {Entropic Time Schedulers for Generative Diffusion Models},
  author    = {Stan\v{c}evi\'{c}, Dejan and Handke, Florian and Ambrogioni, Luca},
  booktitle = {Advances in Neural Information Processing Systems},
  year      = {2025},
  url       = {https://openreview.net/forum?id=EfDIApcjgI}
}

@article{jumper2021,
  title   = {Highly Accurate Protein Structure Prediction with AlphaFold},
  author  = {Jumper, John and Evans, Richard and Pritzel, Alexander and Green, Tim and Figurnov, Michael and Ronneberger, Olaf and Tunyasuvunakool, Kathryn and Bates, Russ and {\v{Z}}{\'\i}dek, Augustin and Potapenko, Anna and Bridgland, Alex and Meyer, Clemens and Kohl, Simon A. A. and Ballard, Andrew J. and Cowie, Andrew and Romera-Paredes, Bernardino and Nikolov, Stanislav and Jain, Rishub and Adler, Jonas and Back, Trevor and Petersen, Stig and Reiman, David and Clancy, Ellen and Zielinski, Michal and Steinegger, Martin and Pacholska, Michalina and Berghammer, Tamas and Bodenstein, Sebastian and Silver, David and Vinyals, Oriol and Senior, Andrew W. and Kavukcuoglu, Koray and Kohli, Pushmeet and Hassabis, Demis},
  journal = {Nature},
  volume  = {596},
  pages   = {583--589},
  year    = {2021},
  doi     = {10.1038/s41586-021-03819-2}
}

@misc{jing2024alphafoldmeetsflowmatching,
  title         = {AlphaFold Meets Flow Matching for Generating Protein Ensembles},
  author        = {Bowen Jing and Bonnie Berger and Tommi Jaakkola},
  year          = {2024},
  eprint        = {2402.04845},
  archiveprefix = {arXiv},
  primaryclass  = {q-bio.BM},
  url           = {https://arxiv.org/abs/2402.04845}
}

@inproceedings{proteina2025,
  title     = {Proteina: Scaling Flow-based Protein Structure Generative Models},
  author    = {Geffner, Tomas and Didi, Kieran and Zhang, Zuobai and Reidenbach, Danny and Cao, Zhonglin and Yim, Jason and Geiger, Mario and Dallago, Christian and Kucukbenli, Emine and Vahdat, Arash and Kreis, Karsten},
  booktitle = {International Conference on Learning Representations},
  year      = {2025}
}

@article{bioemu2024,
  title   = {Scalable Emulation of Protein Equilibrium Ensembles with Generative Deep Learning},
  author  = {Lewis, Sarah and Hempel, Tim and Jim{\'e}nez-Luna, Jos{\'e} and Gastegger, Michael and Xie, Yu and Foong, Andrew Y. K. and Satorras, Victor G. and Abdin, Osama and Veeling, Bastiaan S. and Zaporozhets, Iryna and Chen, Yaoyi and Yang, Soojung and Foster, Adam E. and Schneuing, Arne and Nigam, Jigyasa and Barbero, Federico and Stimper, Vincent and Campbell, Andrew and Yim, Jason and Lienen, Marten and Shi, Yu and Zheng, Shuxin and Schulz, Hannes and Munir, Usman and Sordillo, Roberto and Tomioka, Ryota and Clementi, Cecilia and No{\'e}, Frank},
  journal = {Science},
  volume  = {389},
  number  = {6761},
  year    = {2025},
  doi     = {10.1126/science.adv9817}
}

@article{peptron2025,
  title   = {Advancing Protein Ensemble Predictions Across the Order--Disorder Continuum},
  author  = {Invernizzi, Michele and Bottaro, Sandro and Streit, Julian O. and Trentini, Bruno and Niccol{\`o} Alberto Elia Venanzi and Reidenbach, Danny and Lee, Youhan and Dallago, Christian and Sirelkhatim, Hassan and Jing, Bowen and Airoldi, Fabio and Lindorff-Larsen, Kresten and Fisicaro, Carlo and Tamiola, Kamil},
  journal = {bioRxiv},
  year    = {2025},
  doi     = {10.1101/2025.10.18.680935},
  url     = {https://www.biorxiv.org/content/10.1101/2025.10.18.680935v1},
  note    = {Preprint}
}

@article{stam2024atlas,
  title   = {ATLAS: protein flexibility description from atomistic molecular dynamics simulations},
  author  = {Vander Meersche, Yann and Cretin, Gabriel and Gheeraert, Aria and Gelly, Jean-Christophe and Galochkina, Tatiana},
  journal = {Nucleic Acids Research},
  volume  = {52},
  number  = {D1},
  pages   = {D384--D392},
  year    = {2024},
  doi     = {10.1093/nar/gkad1084}
}

@article{ha2015cameo,
  title   = {Continuous Automated Model EvaluatiOn (CAMEO) complementing the critical assessment of structure prediction in CASP12},
  author  = {Haas, Juergen and Barbato, Alessandro and Behringer, Dario and Studer, Gabriel and Roth, Steven and Bertoni, Martino and Mostaguir, Khaled and Gumienny, Rafal and Schwede, Torsten},
  journal = {Proteins: Structure, Function, and Bioinformatics},
  volume  = {86},
  pages   = {387--398},
  year    = {2018},
  doi     = {10.1002/prot.25431}
}

@inproceedings{cuturi2013sinkhorn,
  title     = {Sinkhorn Distances: Lightspeed Computation of Optimal Transport},
  author    = {Cuturi, Marco},
  booktitle = {Advances in Neural Information Processing Systems},
  volume    = {26},
  year      = {2013}
}

@inproceedings{liu2023rectified,
  title     = {Flow Straight and Fast: Learning to Generate and Transfer Data with Rectified Flow},
  author    = {Liu, Xingchao and Gong, Chengyue and Liu, Qiang},
  booktitle = {International Conference on Learning Representations},
  year      = {2023}
}

@article{chen2023importance,
  title     = {On the Importance of Noise Scheduling for Diffusion Models},
  author    = {Chen, Ting},
  journal   = {arXiv preprint arXiv:2301.10972},
  year      = {2023}
}

@article{lu2022dpmsolverpp,
  title     = {{DPM-Solver++}: Fast Solver for Guided Sampling of Diffusion Probabilistic Models},
  author    = {Lu, Cheng and Zhou, Yuhao and Bao, Fan and Chen, Jianfei and Li, Chongxuan and Zhu, Jun},
  journal   = {arXiv preprint arXiv:2211.01095},
  year      = {2022}
}

@inproceedings{zhang2022deis,
  title     = {Fast Sampling of Diffusion Models with Exponential Integrator},
  author    = {Zhang, Qinsheng and Chen, Yongxin},
  booktitle = {International Conference on Learning Representations},
  year      = {2023}
}

@article{sabour2024alignsteps,
  title     = {Align Your Steps: Optimizing Sampling Schedules in Diffusion Models},
  author    = {Sabour, Amirmojtaba and Fidler, Sanja and Kreis, Karsten},
  journal   = {arXiv preprint arXiv:2404.14507},
  year      = {2024}
}

@article{tsimpos2025lipschitz,
  title     = {Optimal Scheduling of Dynamic Transport},
  author    = {Tsimpos, Panos and Ren, Zhi and Zech, Jakob and Marzouk, Youssef},
  journal   = {arXiv preprint arXiv:2504.14425},
  year      = {2025}
}

@article{gushchin2025ibmd,
  title     = {Inverse Bridge Matching Distillation},
  author    = {Gushchin, Nikita and others},
  journal   = {arXiv preprint},
  year      = {2025}
}

@article{he2024cdbm,
  title     = {Consistency Diffusion Bridge Models},
  author    = {He, Guande and others},
  journal   = {arXiv preprint},
  year      = {2024}
}

@article{stancevic2026information,
  title   = {The Information Dynamics of Generative Diffusion},
  author  = {Stan\v{c}evi\'{c}, Dejan and Ambrogioni, Luca},
  journal = {Entropy},
  volume  = {28},
  number  = {2},
  pages   = {195},
  year    = {2026},
  doi     = {10.3390/e28020195}
}

@inproceedings{tan2026stork,
  title     = {{STORK}: Faster Diffusion and Flow Matching Sampling by Resolving both Stiffness and Structure-Dependence},
  author    = {Tan, Zheng and Wang, Weizhen and Bertozzi, Andrea L. and Ryu, Ernest K.},
  booktitle = {International Conference on Learning Representations},
  year      = {2026},
  url       = {https://openreview.net/forum?id=CeOIVXMl4r}
}

@inproceedings{ai2026jointdistillation,
  title     = {Joint Distillation for Fast Likelihood Evaluation and Sampling in Flow-based Models},
  author    = {Ai, Xinyue and He, Yutong and Gu, Albert and Salakhutdinov, Ruslan and Kolter, J. Zico and Boffi, Nicholas Matthew and Simchowitz, Max},
  booktitle = {International Conference on Learning Representations},
  year      = {2026},
  url       = {https://openreview.net/forum?id=8uZ5UdIul2}
}

@inproceedings{passaro2026stochasticfewstep,
  title     = {Stochastic Few-step Models},
  author    = {Passaro, Romeo and Blasingame, Zander W. and Bronstein, Michael M. and Tong, Alexander},
  booktitle = {ICLR 2026 Workshop on Reasoning and Planning for Large Generative Models},
  year      = {2026},
  url       = {https://openreview.net/forum?id=i31AUZF8kO}
}

@inproceedings{min2026bezierflow,
  title     = {{B{\'e}zierFlow}: Learning {B{\'e}zier} Stochastic Interpolant Schedulers for Few-Step Generation},
  author    = {Min, Yunhong and Koo, Juil and Yoo, Seungwoo and Sung, Minhyuk},
  booktitle = {International Conference on Learning Representations},
  year      = {2026},
  url       = {https://openreview.net/forum?id=PCuDI32xhQ}
}

@inproceedings{liu2025asbs,
  title     = {Adjoint {S}chr{\"o}dinger Bridge Sampler},
  author    = {Liu, Guan-Horng and Choi, Jaemoo and Chen, Yongxin and Miller, Benjamin Kurt and Chen, Ricky T. Q.},
  booktitle = {Advances in Neural Information Processing Systems},
  year      = {2025},
  url       = {https://openreview.net/forum?id=rMhQBlhh4c}
}

@inproceedings{tamogashev2026dataenergy,
  title     = {Data-to-Energy Stochastic Dynamics},
  author    = {Tamogashev, Kirill and Malkin, Nikolay},
  booktitle = {International Conference on Learning Representations},
  year      = {2026},
  url       = {https://openreview.net/forum?id=S1JJyWg1VG}
}

@inproceedings{tang2026branched,
  title     = {Branched {S}chr{\"o}dinger Bridge Matching},
  author    = {Tang, Sophia and Zhang, Yinuo and Tong, Alexander and Chatterjee, Pranam},
  booktitle = {International Conference on Learning Representations},
  year      = {2026},
  url       = {https://openreview.net/forum?id=ctq8BfUXWz}
}

@inproceedings{wyrwal2026topological,
  title     = {Topological Flow Matching},
  author    = {Wyrwal, Kacper and Ceylan, Ismail Ilkan and Tong, Alexander},
  booktitle = {International Conference on Learning Representations},
  year      = {2026},
  url       = {https://openreview.net/forum?id=5CM3ax45Ma}
}

@inproceedings{ksenofontov2025categorical,
  title     = {Categorical {S}chr{\"o}dinger Bridge Matching},
  author    = {Ksenofontov, Grigoriy and Korotin, Aleksandr},
  booktitle = {International Conference on Machine Learning},
  year      = {2025},
  url       = {https://icml.cc/virtual/2025/poster/45290}
}

@inproceedings{irwin2024molflow,
  title     = {Efficient 3D Molecular Generation with Flow Matching and Scale Optimal Transport},
  author    = {Irwin, Ross and Tibo, Alessandro and Janet, Jon Paul and Olsson, Simon},
  booktitle = {ICML Workshop on AI for Science},
  year      = {2024},
  url       = {https://icml.cc/virtual/2024/36820}
}

@inproceedings{miller2024flowmm,
  title     = {{FlowMM}: Generating Materials with Riemannian Flow Matching},
  author    = {Miller, Benjamin Kurt and Chen, Ricky T. Q. and Sriram, Anuroop and Wood, Brandon},
  booktitle = {International Conference on Machine Learning},
  year      = {2024},
  url       = {https://icml.cc/virtual/2024/poster/33890}
}

@inproceedings{haviv2025wasserstein,
  title     = {Wasserstein Flow Matching: Generative Modeling Over Families of Distributions},
  author    = {Haviv, Doron and Pooladian, Aram-Alexandre and Pe'er, Dana and Amos, Brandon},
  booktitle = {International Conference on Machine Learning},
  year      = {2025},
  url       = {https://icml.cc/virtual/2025/poster/45541}
}

@inproceedings{ziebart2008maximum,
    title     = {Maximum Entropy Inverse Reinforcement Learning},
    author    = {Ziebart, Brian D. and Maas, Andrew and Bagnell, J. Andrew and Dey, Anind K.},
    booktitle = {Proceedings of the 23rd {AAAI} Conference on Artificial Intelligence},
    pages     = {1433--1438},
    year      = {2008}
  }

@inproceedings{haarnoja2018soft,
    title     = {Soft Actor-Critic: Off-Policy Maximum Entropy Deep Reinforcement Learning with a Stochastic Actor},
    author    = {Haarnoja, Tuomas and Zhou, Aurick and Abbeel, Pieter and Levine, Sergey},
    booktitle = {Proceedings of the 35th International Conference on Machine Learning},
    series    = {Proceedings of Machine Learning Research},
    volume    = {80},
    pages     = {1861--1870},
    year      = {2018},
    url       = {https://proceedings.mlr.press/v80/haarnoja18b.html}
  }

@article{tishby2000information,
    title   = {The Information Bottleneck Method},
    author  = {Tishby, Naftali and Pereira, Fernando C. and Bialek, William},
    journal = {arXiv preprint physics/0004057},
    year    = {2000},
    url     = {https://arxiv.org/abs/physics/0004057}
  }

@article{bell1995infomax,
    title   = {An Information-Maximization Approach to Blind Separation and Blind Deconvolution},
    author  = {Bell, Anthony J. and Sejnowski, Terrence J.},
    journal = {Neural Computation},
    volume  = {7},
    number  = {6},
    pages   = {1129--1159},
    year    = {1995},
    doi     = {10.1162/neco.1995.7.6.1129}
  }

@inproceedings{chen2016infogan,
    title     = {{InfoGAN}: Interpretable Representation Learning by Information Maximizing Generative Adversarial Nets},
    author    = {Chen, Xi and Duan, Yan and Houthooft, Rein and Schulman, John and Sutskever, Ilya and Abbeel, Pieter},
    booktitle = {Advances in Neural Information Processing Systems},
    volume    = {29},
    year      = {2016},
    url       = {https://papers.nips.cc/paper/6399-infogan-interpretable-representation-learning-by-information-maximizing-generative-adversarial-nets}
  }

@inproceedings{belghazi2018mine,
    title     = {Mutual Information Neural Estimation},
    author    = {Belghazi, Mohamed Ishmael and Baratin, Aristide and Rajeshwar, Sai and Ozair, Sherjil and Bengio, Yoshua and Courville, Aaron and Hjelm, Devon},
    booktitle = {Proceedings of the 35th International Conference on Machine Learning},
    series    = {Proceedings of Machine Learning Research},
    volume    = {80},
    pages     = {531--540},
    year      = {2018},
    url       = {https://proceedings.mlr.press/v80/belghazi18a.html}
  }

@misc{livne2019mim,
    title         = {{MIM}: Mutual Information Machine},
    author        = {Livne, Micha and Swersky, Kevin and Fleet, David J.},
    year          = {2019},
    eprint        = {1910.03175},
    archiveprefix = {arXiv},
    primaryclass  = {cs.LG},
    doi           = {10.48550/arXiv.1910.03175},
    url           = {https://arxiv.org/abs/1910.03175}
  }

@misc{park2024peptidedpo,
    title         = {Improving Inverse Folding for Peptide Design with Diversity-regularized Direct Preference Optimization},
    author        = {Park, Ryan and Hsu, Darren J. and Roland, C. Brian and Korshunova, Maria and Tessler, Chen and Mannor, Shie and Viessmann, Olivia and Trentini,
  Bruno},
    year          = {2024},
    eprint        = {2410.19471},
    archiveprefix = {arXiv},
    primaryclass  = {cs.LG},
    doi           = {10.48550/arXiv.2410.19471},
    url           = {https://arxiv.org/abs/2410.19471}
  }


\appendix

\newpage
\begin{center}
{\Large\bfseries Appendix}
\end{center}
\vspace{1em}

\section{A guided derivation of the conditional--marginal rate}
\label{app:full_derivations}

\paragraph{Marginal entropy rate.}
\label{app:marginal_entropy}

We assumed in the main text that the density and vector field are smooth enough for integration by parts and that boundary terms vanish. We now walk through the calculation. Consider first the marginal entropy
\begin{equation}
  \Ent(X_t)=-\int p_t(\vx)\log p_t(\vx)\,\dd \vx,
\end{equation}
and differentiate with respect to time:
\begin{align}
  \frac{\dd}{\dd t}\Ent(X_t)
  &=
  -\int \partial_t p_t(\vx)\log p_t(\vx)\,\dd\vx
  -\int p_t(\vx)\frac{\partial_t p_t(\vx)}{p_t(\vx)}\,\dd\vx \\
  &=
  -\int \partial_t p_t(\vx)\log p_t(\vx)\,\dd\vx
  -\frac{\dd}{\dd t}\int p_t(\vx)\,\dd\vx \\
  &=
  -\int \partial_t p_t(\vx)\log p_t(\vx)\,\dd\vx.
\end{align}
The second term vanishes because $p_t$ remains normalized. Next, use the continuity equation,
\begin{align}
  \frac{\dd}{\dd t}\Ent(X_t)
  &=
  \int \diver(p_t(\vx)\bar{\vv}_t(\vx))\log p_t(\vx)\,\dd\vx \\
  &=
  -\int p_t(\vx)\bar{\vv}_t(\vx)^\top \nabla_{\vx}\log p_t(\vx)\,\dd\vx \\
  &=
  -\int \bar{\vv}_t(\vx)^\top \nabla_{\vx}p_t(\vx)\,\dd\vx \\
  &=
  \int p_t(\vx)\diver \bar{\vv}_t(\vx)\,\dd\vx \\
  &=
  \E_{X_t}\!\left[\diver\bar{\vv}_t(X_t)\right].
\end{align}

\paragraph{Conditional entropy rate.}
\label{app:conditional_entropy}

The conditional calculation repeats the same story while holding the bridge condition fixed. For each condition $Z=z$, the conditional density follows its own continuity equation, so
\begin{equation}
  \frac{\dd}{\dd t}\Ent(X_t\mid Z=z)
  =
  \E_{X_t\mid Z=z}\!\left[\diver \vv_t(X_t\mid z)\right].
\end{equation}
Averaging over $Z$ gives
\begin{equation}
  \frac{\dd}{\dd t}\Ent(X_t\mid Z)
  =
  \E_{Z,X_t\mid Z}\!\left[\diver \vv_t(X_t\mid Z)\right].
\end{equation}
Now take into consideration the entropy chain rule,
\begin{equation}
  \Ent(Z\mid X_t)=\Ent(X_t\mid Z)+\Ent(Z)-\Ent(X_t),
\end{equation}
and subtract the marginal entropy rate derived above. We obtain
\begin{equation}
  \frac{\dd}{\dd t}\Ent(Z\mid X_t)
  =
  \E_{Z,X_t\mid Z}\!\left[\diver \vv_t(X_t\mid Z)\right]
  -
  \E_{X_t}\!\left[\diver \bar{\vv}_t(X_t)\right].
\end{equation}


\paragraph{Score representation.}
\label{app:score_derivation}

The divergence form is the estimator used in this paper, but the score form clarifies what is being measured. The marginal field is
\begin{equation}
  \bar{\vv}_t(\vx)=\int p_t(z\mid \vx)\vv_t(\vx\mid z)\,\dd z.
\end{equation}
Taking divergence,
\begin{align}
  \diver \bar{\vv}_t(\vx)
  &=
  \int p_t(z\mid\vx)\diver \vv_t(\vx\mid z)\,\dd z
  +
  \int \nabla_{\vx}p_t(z\mid\vx)^\top \vv_t(\vx\mid z)\,\dd z.
\end{align}
Multiply by $p_t(\vx)$ and integrate over $\vx$. The first term becomes the expected conditional divergence. For the second term, Bayes' rule gives
\begin{equation}
  \nabla_{\vx}\log p_t(z\mid\vx)
  =
  \nabla_{\vx}\log p_t(\vx\mid z)
  -
  \nabla_{\vx}\log p_t(\vx).
\end{equation}
Therefore
\begin{align}
  \E_{X_t}\!\left[\diver\bar{\vv}_t(X_t)\right]
  &=
  \E_{Z,X_t\mid Z}\!\left[\diver\vv_t(X_t\mid Z)\right] \nonumber \\
  &\quad+
  \E_{Z,X_t\mid Z}\!\left[
  (\nabla\log p_t(X_t\mid Z)-\nabla\log p_t(X_t))^\top
  \vv_t(X_t\mid Z)
  \right].
\end{align}
Rearranging yields \Cref{eq:score_form}. This shows that the rate measures the conditional field against the gap between the conditional score and the marginal score.

\paragraph{Gaussian conditional vector field.}
\label{app:gaussian_derivation}

The Brownian-bridge calculation in the main text is a special case of a Gaussian conditional path. Let
\begin{equation}
  X_t=\mu(z,t)+\sigma(z,t)\veps,\qquad \veps\sim\mathcal{N}(0,I).
\end{equation}
For fixed noise $\veps$,
\begin{equation}
  \partial_t X_t=\partial_t\mu(z,t)+\partial_t\sigma(z,t)\veps.
\end{equation}
Since $\veps=(X_t-\mu(z,t))/\sigma(z,t)$,
\begin{equation}
  \vv_t(X_t\mid z)=\partial_t\mu(z,t)+\partial_t\sigma(z,t)\frac{X_t-\mu(z,t)}{\sigma(z,t)}.
\end{equation}
The Gaussian score is
\begin{equation}
  \nabla_{X_t}\log p_t(X_t\mid z)=-\frac{X_t-\mu(z,t)}{\sigma^2(z,t)}.
\end{equation}
Substitution gives
\begin{equation}
  \vv_t(X_t\mid z)
  =
  \partial_t\mu(z,t)
  -
  \sigma(z,t)\partial_t\sigma(z,t)
  \nabla_{X_t}\log p_t(X_t\mid z).
\end{equation}

\section{Brownian bridge details}
\label{app:bb_details}

\subsection{Probability-flow field}

For $m_t=(1-t)x_0+tx_1$ and $\sigma(t)=\sigma_0\sqrt{t(1-t)}$,
\begin{equation}
  \partial_t m_t=x_1-x_0.
\end{equation}
The time derivative of the standard deviation is
\begin{align}
  \partial_t\sigma(t)
  &=
  \sigma_0\frac{1}{2\sqrt{t(1-t)}}(1-2t),
\end{align}
so
\begin{equation}
  \sigma(t)\partial_t\sigma(t)
  =
  \sigma_0\sqrt{t(1-t)}
  \frac{\sigma_0(1-2t)}{2\sqrt{t(1-t)}}
  =
  \frac{\sigma_0^2(1-2t)}{2}.
\end{equation}
The conditional score is
\begin{equation}
  \nabla_{\vx}\log p_t(\vx\mid x_0,x_1)
  =
  -\frac{\vx-m_t}{\sigma_0^2t(1-t)}.
\end{equation}
Substituting into the Gaussian field,
\begin{align}
  \vv_t(\vx\mid x_0,x_1)
  &=
  x_1-x_0
  -
  \frac{\sigma_0^2(1-2t)}{2}
  \left[-\frac{\vx-m_t}{\sigma_0^2t(1-t)}\right] \\
  &=
  x_1-x_0
  +
  \frac{1-2t}{2t(1-t)}(\vx-m_t).
\end{align}

\subsection{SDE drift and the factor of two}
\label{app:sde_ode_details}

The Brownian-bridge SDE drift is
\begin{equation}
  u^o_t(\vx\mid x_0,x_1)
  =
  x_1-x_0
  +
  \frac{1-2t}{t(1-t)}(\vx-m_t).
\end{equation}
Comparing with the probability-flow field gives
\begin{equation}
  u^o_t(\vx\mid x_0,x_1)
  =
  2\vv_t(\vx\mid x_0,x_1)-(x_1-x_0).
\end{equation}
The factor of two arises from the chain-rule term $\sigma(t)\partial_t\sigma(t)$ in the probability-flow derivation. If one applies the marginal probability-flow transformation to the conditional objects, the score correction cancels the bridge contraction term incorrectly:
\begin{align}
  u^o_t+\sigma_0^2(1-2t)\nabla_{\vx}\log p_t(\vx\mid x_0,x_1)
  &=
  x_1-x_0
  +
  \frac{1-2t}{t(1-t)}(\vx-m_t)
  -
  \frac{1-2t}{t(1-t)}(\vx-m_t) \\
  &=
  x_1-x_0.
\end{align}
The result differs from \Cref{eq:bb_ode_field}. The probability-flow relation is valid for marginal fields after endpoint mixing, not for this conditional conversion.

\section{Interpolation profiles and bridge specificity}
\label{app:interpolation_profiles}

We raised an important scope question: is the U-shape a generic property of every interpolation, or a bridge-specific calculation? The answer is bridge-specific. Different interpolants produce different analytic divergence profiles, and the learned field can deviate from the exact conditional formula. This is why the main text treats the Brownian-bridge profile as an anchor and the conditional--marginal estimator as the general object.

\begin{table}[H]
  \caption{Analytic divergence profiles for common conditional interpolants. Here $d$ is the ambient dimension. The table illustrates that different path choices lead to different schedule signals.}
  \label{tab:interpolation_profiles}
  \centering
  \small
  \begin{tabular}{P{0.24\linewidth}P{0.42\linewidth}P{0.22\linewidth}}
    \toprule
    Interpolation & Exact conditional divergence & Qualitative profile \\
    \midrule
    Linear optimal-transport interpolation & $0$ for the exact linear conditional field & Flat before learning error \\
    Variance-preserving diffusion path & $d\,t/(1-t^2)$ under the standard VP parameterization & Monotone boundary growth \\
    Cosine diffusion path & $d\pi\tan(\pi t/2)/2$ under the cosine noise parameterization & Monotone boundary growth \\
    Brownian bridge & $d(1-2t)/[2t(1-t)]$ for the probability-flow field & Symmetric U-shape in magnitude \\
    \bottomrule
  \end{tabular}
\end{table}

For Schr\"odinger bridges represented as mixtures of Brownian bridges, the conditional Brownian divergence in \Cref{eq:bb_divergence} does not depend on the endpoint pair $(x_0,x_1)$ or on the spatial position $\vx$. Therefore
\begin{equation}
  \E_{(X_0,X_1),X_t\mid X_0,X_1}
  \left[\diver\vv_t(X_t\mid X_0,X_1)\right]
  =
  d\,\frac{1-2t}{2t(1-t)}
\end{equation}
for the conditional term of the Brownian reference bridge. The marginal term in \Cref{eq:cm_identity} remains model- and coupling-dependent, which is why the full bridge rate is not reduced to the conditional term alone.

\section{Boundary concentration ratio}
\label{app:bcr}

We introduced boundary concentration ratio after observing that many useful schedules differ mainly in where they place density near $t=0$ and $t=1$. The entropy rate gives a full curve, but comparisons across many schedules are easier when one also has a scalar diagnostic. The diagnostic should answer a narrow question: how much more schedule mass lies near the two endpoints than would lie there under a uniform grid?

Take into consideration a normalized schedule density $q(t)$ on $[0,1]$ with $\int_0^1q(t)\dd t=1$. For a boundary width $\epsilon\in(0,1/2)$, define the endpoint mass
\begin{equation}
  M_\epsilon(q)
  =
  \int_0^\epsilon q(t)\,\dd t+\int_{1-\epsilon}^1 q(t)\,\dd t.
\end{equation}
The same two endpoint intervals have mass $2\epsilon$ under a uniform grid density. We therefore define
\begin{equation}
  \bcr_\epsilon(q)
  =
  \frac{M_\epsilon(q)}{2\epsilon}.
  \label{eq:bcr}
\end{equation}
A uniform grid has $\bcr_\epsilon=1$. Values larger than one mean that the schedule assigns more node density to the boundary windows than uniform spacing. Values below one mean that the schedule avoids the endpoints relative to the uniform grid.

For a discrete grid, one can estimate $q$ from the interval widths or from the inverse-CDF density used to construct the grid. The density-based version is preferable because it is less sensitive to whether the endpoints themselves are counted as nodes. The assumptions behind BCR are simple: the same $\epsilon$ must be used for all schedules, the grid must be monotone, and the diagnostic should be computed on the schedule density rather than on downstream metrics.

BCR is useful beyond this paper because it separates a geometric property of the grid from sample quality. A high BCR can indicate endpoint stiffness, bridge contraction, or a raw entropy singularity. It can also warn that tempering may be needed, since over-resolving a singular endpoint can waste low-NFE budget. BCR is therefore a schedule diagnostic and not a main claim. A model can have high boundary concentration and still perform poorly if the estimator is mismatched or if the training objective favors another grid.

\section{ODE error derivation}
\label{app:ode_error}

For an order-$p$ one-step method, suppose the local truncation error on a step of width $h$ is bounded by
\begin{equation}
  \|e_{k+1}\|\le C(t_k)h_k^{p+1}+O(h_k^{p+2}).
\end{equation}

Let $\rho(t)$ be a continuous density of grid points. Locally $h(t)\approx 1/(N\rho(t))$. Ignoring constants independent of $\rho$, a continuous proxy for total error is
\begin{equation}
  \mathcal{E}[\rho]=\int_0^1 C(t)\rho(t)^{-p}\,\dd t,
  \qquad
  \int_0^1\rho(t)\,\dd t=1.
\end{equation}
The Lagrangian is
\begin{equation}
  \mathcal{L}[\rho]
  =
  \int_0^1 C(t)\rho(t)^{-p}\,\dd t
  +
  \lambda\left(\int_0^1\rho(t)\,\dd t-1\right).
\end{equation}
The stationarity condition is
\begin{equation}
  -pC(t)\rho(t)^{-p-1}+\lambda=0,
\end{equation}
which gives
\begin{equation}
  \rho(t)\propto C(t)^{1/(p+1)}.
\end{equation}
Entropy-rate scheduling uses $|\dot{\Ent}(Z\mid X_t)|$ as a measurable proxy for $C(t)$.

\section{Algorithms}
\label{app:algorithms}

This appendix states the inference-time pipeline as algorithms. The algorithms do not assume a particular implementation. They separate the calibration stage, where the entropy-rate curve is estimated once, from the sampling stage, where all schedulers are compared at matched NFE.

\begin{algorithm}[H]
  \caption{Calibration Set and Time Mesh Construction}
  \label{alg:calibration}
  \begin{algorithmic}[1]
    \REQUIRE Calibration conditions $\{z_i\}_{i=1}^n$, time range $[\epsilon,1-\epsilon]$, number of mesh points $M$
    \ENSURE Calibration states $\{x_{ij}\}$ and paired conditions $\{z_i,s_j\}$
    \STATE Choose a monotone mesh $0<\epsilon=s_1<\cdots<s_M=1-\epsilon<1$
    \FOR{$i=1,\ldots,n$}
      \STATE Draw or retrieve the bridge condition $z_i$
      \FOR{$j=1,\ldots,M$}
        \STATE Construct the state $x_{ij}\sim p_{s_j}(\cdot\mid z_i)$ when the conditional simulator is available
        \STATE Otherwise, use a model trajectory state paired with its condition and time
      \ENDFOR
    \ENDFOR
    \RETURN Paired calibration set $\{(x_{ij},z_i,s_j)\}$
  \end{algorithmic}
\end{algorithm}

\begin{algorithm}[H]
  \caption{Hutchinson Divergence of a Vector Field}
  \label{alg:hutchinson}
  \begin{algorithmic}[1]
    \REQUIRE Vector field $f_t$, state $x$, time $t$, probes $\{\vu_\ell\}_{\ell=1}^m$ with $\E[\vu_\ell\vu_\ell^\top]=I$
    \ENSURE Estimate of $\diver f_t(x)$
    \STATE $a \leftarrow 0$
    \FOR{$\ell=1,\ldots,m$}
      \STATE Evaluate the Jacobian-vector product $w_\ell \leftarrow \nabla_x f_t(x)\vu_\ell$
      \STATE $a \leftarrow a + \vu_\ell^\top w_\ell$
    \ENDFOR
    \RETURN $a/m$
  \end{algorithmic}
\end{algorithm}

\begin{algorithm}[H]
  \caption{Conditional--Marginal Entropy-Rate Estimation}
  \label{alg:cm_estimator}
  \begin{algorithmic}[1]
    \REQUIRE Calibration set $\{(x_{ij},z_i,s_j)\}$, conditional field $\vv_t(\cdot\mid z)$, marginal field $\bar{\vv}_t$, probes per state $m$
    \ENSURE Estimated rate values $\{\hat r(s_j)\}_{j=1}^M$ and uncertainty estimates
    \FOR{$j=1,\ldots,M$}
      \STATE $d_{j1},\ldots,d_{jn} \leftarrow$ empty
      \FOR{$i=1,\ldots,n$}
        \STATE Draw shared probes $\{\vu_{\ell ij}\}_{\ell=1}^m$
        \IF{analytic conditional divergence is available}
          \STATE $c_{ij}\leftarrow \diver\vv_{s_j}(x_{ij}\mid z_i)$ from the analytic formula
        \ELSE
          \STATE $c_{ij}\leftarrow$ \Cref{alg:hutchinson} applied to $\vv_{s_j}(\cdot\mid z_i)$ at $x_{ij}$
        \ENDIF
        \STATE $m_{ij}\leftarrow$ \Cref{alg:hutchinson} applied to $\bar{\vv}_{s_j}$ at $x_{ij}$ with the same probes
        \STATE $d_{ji}\leftarrow c_{ij}-m_{ij}$
      \ENDFOR
      \STATE $\hat r(s_j)\leftarrow \left|n^{-1}\sum_{i=1}^n d_{ji}\right|$
      \STATE $\widehat{\mathrm{se}}(s_j)\leftarrow$ standard error of $\{d_{ji}\}_{i=1}^n$
    \ENDFOR
    \RETURN Rate curve $\hat r$ and uncertainty curve $\widehat{\mathrm{se}}$
  \end{algorithmic}
\end{algorithm}

\begin{algorithm}[H]
  \caption{Rate Postprocessing and Grid Construction}
  \label{alg:grid}
  \begin{algorithmic}[1]
    \REQUIRE Mesh $\{s_j\}_{j=1}^M$, estimated rate $\hat r$, transform $\phi\in\{r,\log(1+r)\}$, sampling steps $N$
    \ENSURE Monotone grid $0=t_0<t_1<\cdots<t_N=1$
    \STATE Replace non-finite values by local interpolation and clip $\hat r$ to a small positive floor
    \STATE Smooth the clipped curve only on the calibration mesh, preserving endpoint ordering
    \STATE $a_j\leftarrow \phi(\hat r(s_j))$ for $j=1,\ldots,M$
    \STATE Normalize $a_j$ into a density $q_j$ by numerical quadrature
    \STATE Compute the cumulative curve $Q(s_j)\approx\int_0^{s_j}q(u)\dd u$
    \STATE Enforce monotonicity of $Q$ and set $Q(0)=0$, $Q(1)=1$
    \FOR{$k=0,\ldots,N$}
      \STATE Set $t_k\leftarrow Q^{-1}(k/N)$ by monotone interpolation
    \ENDFOR
    \RETURN Grid $\{t_k\}_{k=0}^N$ and density summary $q$
  \end{algorithmic}
\end{algorithm}

\begin{algorithm}[H]
  \caption{Matched-NFE Evaluation Protocol}
  \label{alg:evaluation}
  \begin{algorithmic}[1]
    \REQUIRE Model sampler, scheduler family $\mathcal{S}$, NFE budgets $\mathcal{N}$, seeds $\mathcal{B}$, evaluation metrics $\mathcal{M}$
    \ENSURE Matched comparison table across schedulers
    \FOR{scheduler $S\in\mathcal{S}$}
      \FOR{budget $N\in\mathcal{N}$}
        \STATE Build the grid for $S$ with exactly $N$ intervals
        \FOR{seed $b\in\mathcal{B}$}
          \STATE Generate samples with the same model, solver family, and NFE budget
          \STATE Evaluate all metrics in $\mathcal{M}$ on the generated samples
        \ENDFOR
      \ENDFOR
    \ENDFOR
    \STATE Aggregate means, standard deviations, and paired deltas relative to selected baselines
    \RETURN Matched-NFE metric table and per-sample records
  \end{algorithmic}
\end{algorithm}

\begin{algorithm}[H]
  \caption{Paired Bootstrap Confidence Intervals}
  \label{alg:bootstrap}
  \begin{algorithmic}[1]
    \REQUIRE Per-unit metric records for scheduler $A$ and baseline $B$, bootstrap draws $R$, confidence level $1-\alpha$
    \ENSURE Mean paired difference and percentile interval
    \STATE Match records by evaluation unit, seed, dataset split, and NFE budget
    \FOR{$r=1,\ldots,R$}
      \STATE Resample matched units with replacement
      \STATE Compute the paired difference $\Delta_r=\mathrm{metric}(A)-\mathrm{metric}(B)$ on the resample
    \ENDFOR
    \STATE Report $\bar{\Delta}$ on the original matched records
    \STATE Report percentiles $\alpha/2$ and $1-\alpha/2$ of $\{\Delta_r\}_{r=1}^R$
    \RETURN Paired estimate and confidence interval
  \end{algorithmic}
\end{algorithm}

\begin{algorithm}[H]
  \caption{Boundary-Concentration Diagnostic}
  \label{alg:bcr}
  \begin{algorithmic}[1]
    \REQUIRE Schedule density $q(t)$ or grid density estimate, boundary width $\epsilon\in(0,1/2)$
    \ENSURE Boundary concentration ratio $\bcr_\epsilon$
    \STATE Estimate boundary mass $M_\epsilon=\int_0^\epsilon q(t)\dd t+\int_{1-\epsilon}^1q(t)\dd t$
    \STATE Normalize by the uniform-grid boundary mass $2\epsilon$
    \STATE $\bcr_\epsilon\leftarrow M_\epsilon/(2\epsilon)$
    \RETURN $\bcr_\epsilon$
  \end{algorithmic}
\end{algorithm}

\section{Computational notes}
\label{app:compute}

\begin{table}[H]
  \caption{Complexity of grid construction. Sampling cost at fixed NFE is unchanged after the grid has been constructed.}
  \label{tab:complexity}
  \centering
  \small
  \begin{tabular}{P{0.22\linewidth}P{0.24\linewidth}P{0.22\linewidth}P{0.20\linewidth}}
    \toprule
    Method & Calibration compute & Extra memory & TFLOPS/byte note \\
    \midrule
    Analytic bridge rate & $O(M)$ scalar evaluations & $O(N)$ grid & Negligible \\
    Exact divergence & $O(Mnd)$ derivative products & Large Jacobian or repeated products & Usually impractical \\
    Hutchinson rate & $O(Mnm)$ derivative-vector products & Probe buffers plus activations & Kernel-dependent \\
    Reused entropy grid & None during sampling & $O(N)$ grid & Same NFE as baseline \\
    \bottomrule
  \end{tabular}
\end{table}

Here $M$ is the number of calibration times, $n$ the number of calibration states, $m$ the number of Hutchinson probes, $d$ the dimension, and $N$ the number of sampling steps. Hardware-level TFLOPS/byte depends on kernels, precision, and hardware, so the robust comparison is symbolic.

\section{Hypotheses and conclusions}
\label{app:hypotheses}

We guided the work with three overall questions. The first asks whether the bridge calculation predicts a distinctive information geometry. The second asks whether learned two-dimensional probes preserve that geometry after the vector field is trained. The third asks whether cond--marg estimator and its grid help interpret high-dimensional low-NFE behavior without turning an allocation diagnostic into an endpoint-quality claim.


\begin{table}[H]
  \caption{Questions that guided the evidence and the corresponding conclusions.}
  \label{tab:hypotheses}
  \centering
  \small
  \begin{tabular}{P{0.28\linewidth}P{0.42\linewidth}P{0.18\linewidth}}
    \toprule
    Guiding question & Evidence & Conclusion \\
    \midrule
    What quantity should a bridge schedule measure? & \Cref{eq:cm_identity} and \Cref{eq:bb_divergence} show that the bridge rate is a conditional--marginal divergence contrast, with a U-shaped Brownian-bridge conditional term. & Use a two-term bridge rate, not a single-field proxy. \\
    Does this geometry appear in high dimension? & AlphaFlow conditional--marginal profiles show a boundary-heavy U-shape, while EDM gives a different profile and therefore a useful mismatch case. & The rate is model-dependent and must be measured or derived. \\
    Can the measured rate guide low-NFE sampling? & EDM five-step FID improves with log-tempered entropy, and AlphaFlow allocation diagnostics align with the bridge calculation while endpoint pLDDT remains training-objective dependent. & Entropy is a grid signal, not a universal endpoint-quality guarantee. \\
    \bottomrule
  \end{tabular}
\end{table}

\section{Additional method comparison}
\label{app:comparison}

\begin{table}[H]
  \caption{Comparison with related acceleration and scheduling families.}
  \label{tab:comparison}
  \centering
  \small
  \begin{tabular}{P{0.24\linewidth}P{0.30\linewidth}P{0.32\linewidth}}
    \toprule
    Family & Main object & Relation to this work \\
    \midrule
    Heuristic grids & Fixed analytic time maps & Strong baselines, but not path-measured \\
    Align Your Steps & Optimized discrete schedule & Empirical optimization rather than entropy identity \\
    Lipschitz and transport scheduling & Smoothness or transport-cost proxy & Related to error constants, no conditional--marginal bridge decomposition \\
    Fast ODE solvers & Numerical update formula & Complementary to grid choice \\
    Distillation & New few-step model & Requires training, unlike inference-time grid construction \\
    Bridge and flow coupling & Training path geometry & Changes the learned path rather than measuring the grid density \\
    \bottomrule
  \end{tabular}
\end{table}

    \begin{table}[H]
      \caption{Closest entropy-based scheduler comparison. The bridge-specific change is the conditional--marginal divergence contrast.}
      \label{tab:stancevic_appendix}
      \centering
      \small
      \setlength{\tabcolsep}{3pt}
      \begin{tabular}{@{}P{0.14\linewidth}@{\hspace{0.02\linewidth}}P{0.39\linewidth}@{\hspace{0.02\linewidth}}P{0.39\linewidth}@{}}
        \toprule
        Aspect & Entropic diffusion time & Conditional--marginal bridge time \\
        \midrule
        Object & One-endpoint diffusion process & Conditional bridge or flow path \\
        Signal & Entropy of data given noised state & $\dot{\Ent}(Z\mid X_t)$ from conditional minus marginal divergence \\
        Estimator & Diffusion loss or single-field formula & Matched conditional and marginal divergence estimates \\
        Shape & Often asymmetric in noise time & Symmetric U-shape for Brownian bridges \\
        Use & Inference-time reparameterization & Inference-time grid for bridge and flow samplers \\
        \bottomrule
      \end{tabular}
    \end{table}

\section{Protein endpoint confidence metrics}
\label{app:plddt_metric}

Protein endpoint quality is reported with predicted local distance difference test (plDDT), a per-residue confidence estimate associated with distance agreements \citep{jumper2021, mariani2013lddt}. The score is usually reported on a 0--100 scale. Higher values indicate higher local confidence. We average plDDT over generated endpoint structures and use it as an endpoint confidence \emph{proxy}. This is useful because it is cheap, standardized, and sensitive to local structural plausibility. However, it is not the same object as the cond--marg entropy rate. plDDT does not measure where the sampler placed time steps, and it can depend on the target protein, sequence length, flexible or disordered regions, model calibration, and the evaluator. Hence, pLDDT is interpreted together with diversity and trajectory diagnostics, not as direct proof that one grid is the correct entropy grid.

\paragraph{Protein evaluation funnel.}
\label{app:protein_eval_funnel}
CAMEO22 and ATLAS are used here as evaluation sources, not as evidence that the entropy grid was optimized during AlphaFlow training. The Sankey diagram in \Cref{fig:alphaflow_testing_sankey_appendix} separates the available benchmark proteins from the subset that passed the native-score testing path. The count table behind the figure reports 48 available ATLAS proteins and 72 available CAMEO proteins; 115 proteins were tested after filtering (45 ATLAS, 70 CAMEO), with give large proteins not tested due to system timeouts. These counts describe evaluation coverage. They should be read separately from endpoint pLDDT, which remains a confidence proxy, and separately from the cond--marg rate, which is the schedule-geomtry signal.

\section{Toy and synthetic data}
\label{app:figures}

The synthetic appendix is included to show what the theory predicts before the high-dimensional stress tests. These plots are not used as the main evidence. They help the reader see why a bridge can ask for a non-uniform grid even when the downstream model is simple.

    \begin{figure}[!t]
      \centering
      \begin{minipage}[t]{0.49\linewidth}
        \centering
        \adjustbox{max width=\linewidth, max height=0.32\textheight, valign=t}{%
          \includegraphics{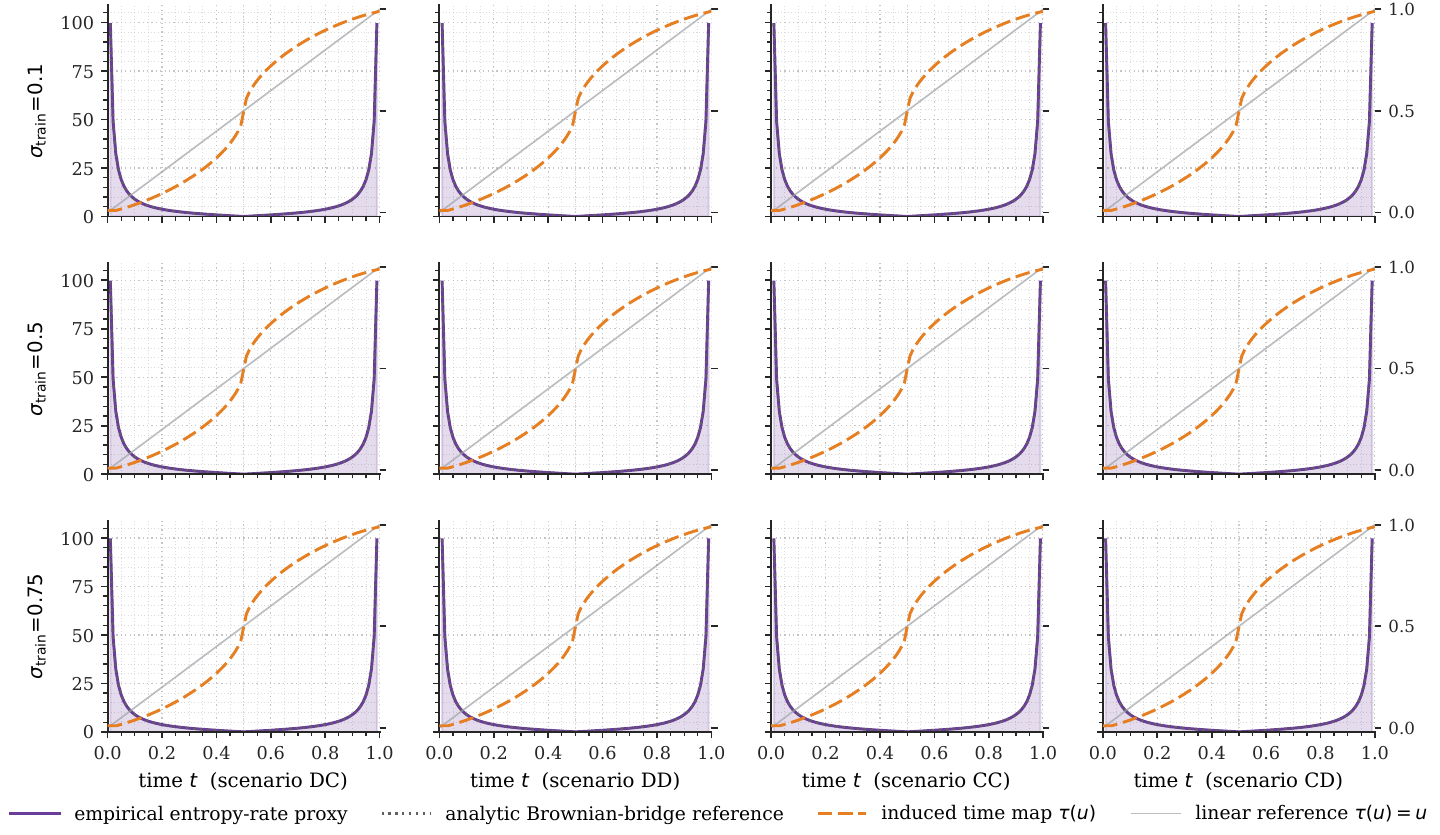}%
        }
      \end{minipage}
      \hfill
      \begin{minipage}[t]{0.49\linewidth}
        \centering
        \adjustbox{max width=\linewidth, max height=0.32\textheight, valign=t}{%
          \includegraphics{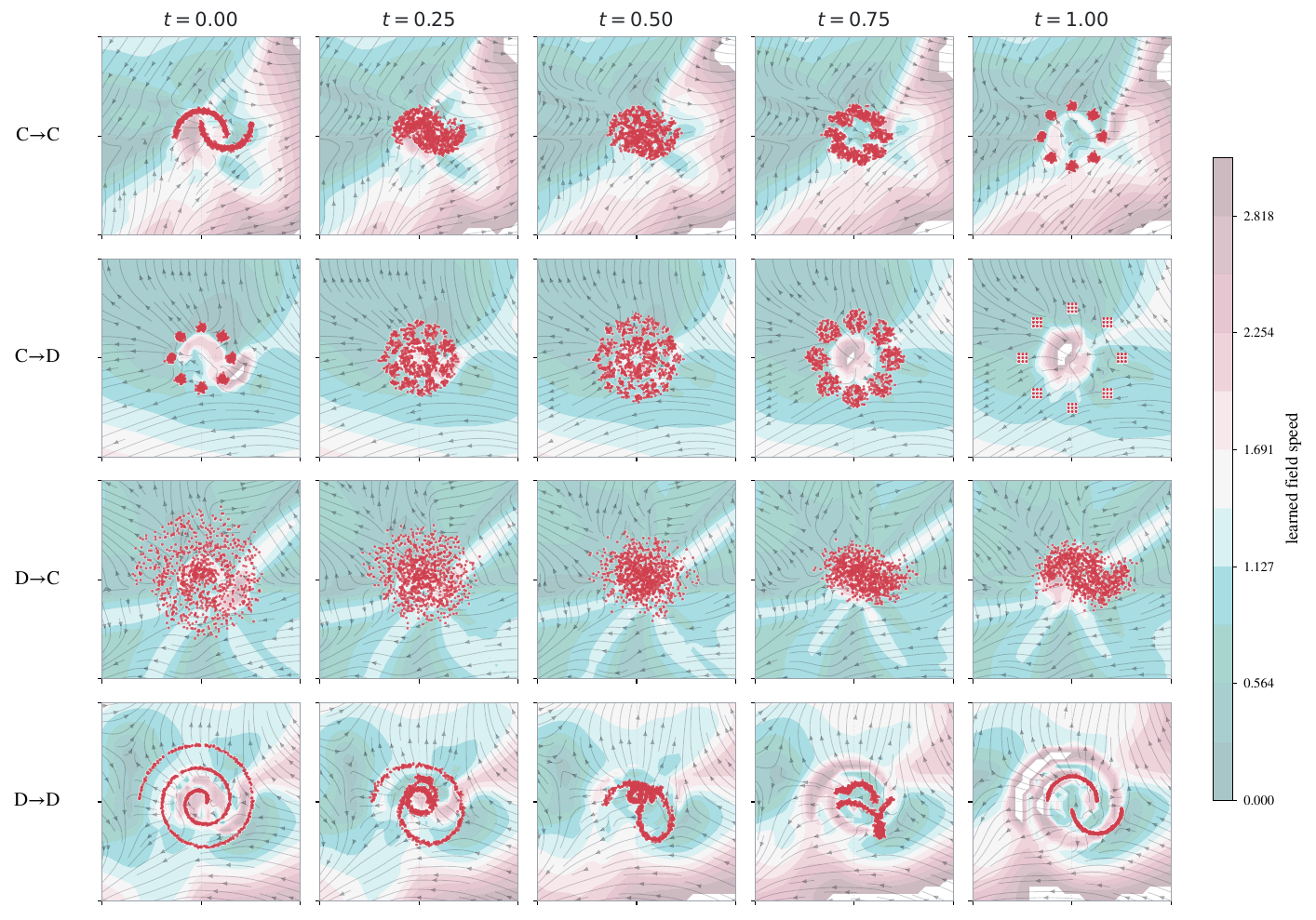}%
        }
      \end{minipage}
      \caption{Controlled transport setup and entropy-induced grids. Left: entropy-rate profiles and induced inverse-CDF time grids show the core mechanism: measure where the path carries information, normalize that rate, and place time nodes by the corresponding cumulative distribution. Right: synthetic transport scenarios used as controlled probes, contrasting continuous--continuous, continuous--discrete, discrete--continuous, and discrete--discrete endpoint structure.}
      \label{fig:controlled_transport_setup_appendix}
    \end{figure}




\begin{figure}[H]
  \centering
  \includegraphics[width=0.78\linewidth]{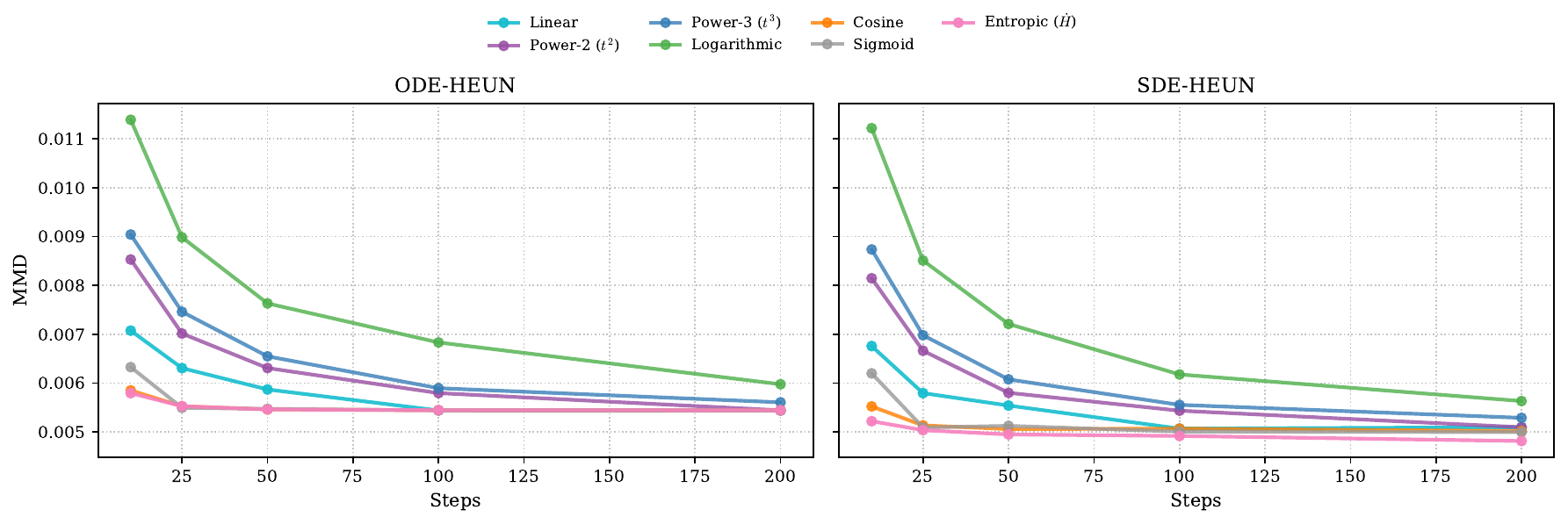}
  \caption{Scenario-level comparison on the toy transports. The figure is included to make the synthetic evidence explicit: the schedule effect depends on the endpoint geometry, and the role of the toy data is to isolate this dependence before moving to EDM and AlphaFlow.}
  \label{fig:scenario_comparison_appendix}
\end{figure}

\begin{figure}[H]
  \centering
  \includegraphics[width=0.78\linewidth]{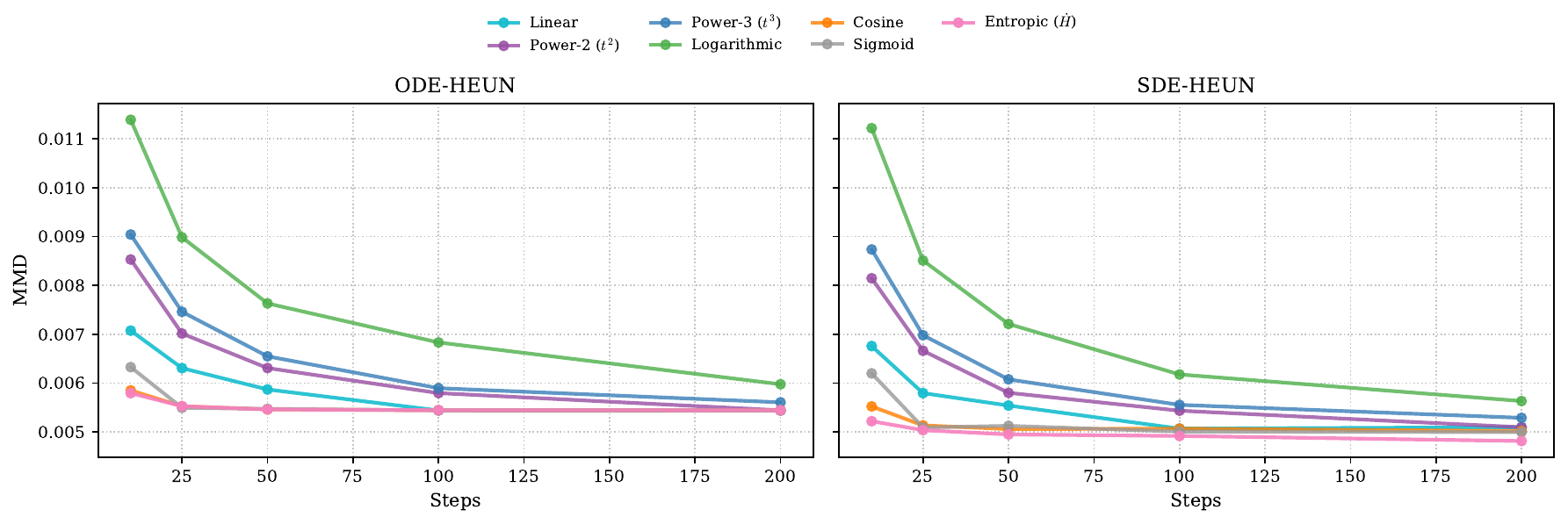}
  \caption{Initial ODE/SDE probing on synthetic data. The main observation is that low-NFE behavior is more sensitive to grid placement than high-NFE behavior, and deterministic integration exposes this sensitivity more directly than stochastic integration.}
  \label{fig:mmd_steps_appendix}
\end{figure}

\begin{figure}[H]
  \centering
  \includegraphics[width=0.78\linewidth]{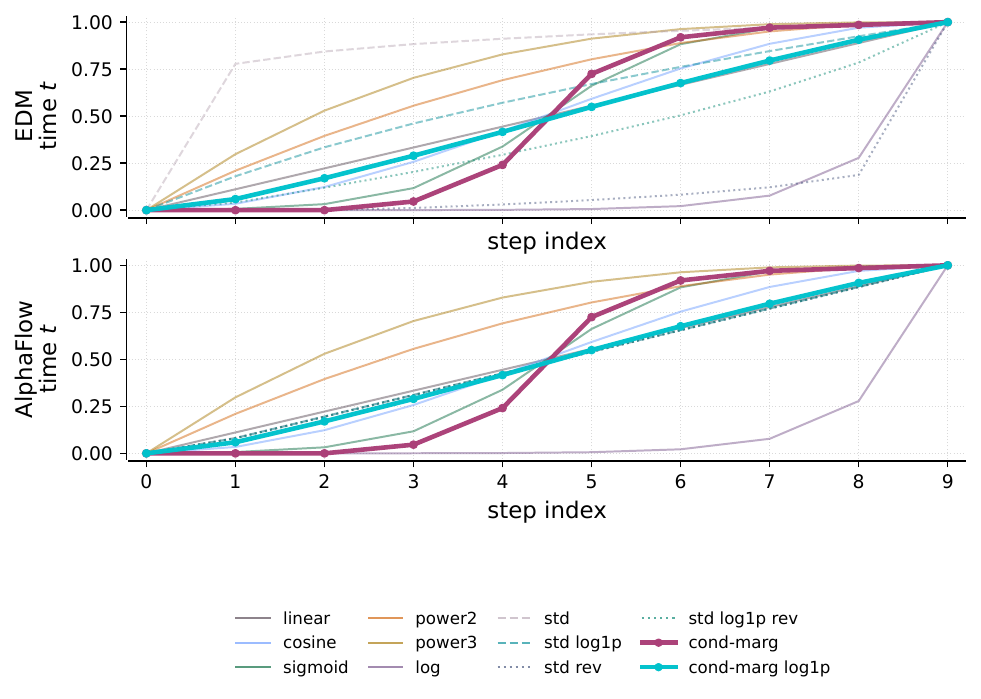}
  \caption{Reference scheduler shapes at $\sigma=0.5$. This plot is kept as a reference for how common heuristic grids distribute nodes relative to the bridge-inspired profiles. It is not used to claim that one heuristic is universally best.}
  \label{fig:scheduler_shapes_appendix}
\end{figure}

  \section{Evaluation Grids}
  \label{app:evaluation_grids}

The following figures are grid and trajectory diagnostics referenced from the main text.

\begin{figure}[t]
  \centering
  \begin{subfigure}[t]{0.49\linewidth}
    \centering
    \includegraphics[width=\linewidth]{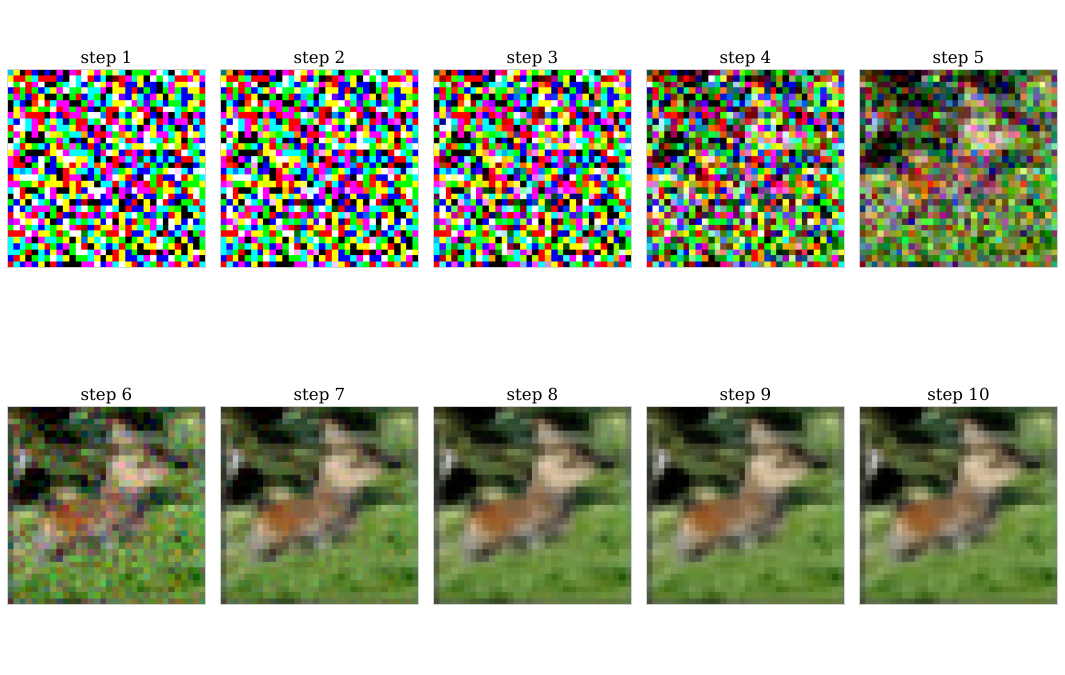}
    \caption{Linear.}
  \end{subfigure}
  \hfill
  \begin{subfigure}[t]{0.49\linewidth}
    \centering
    \includegraphics[width=\linewidth]{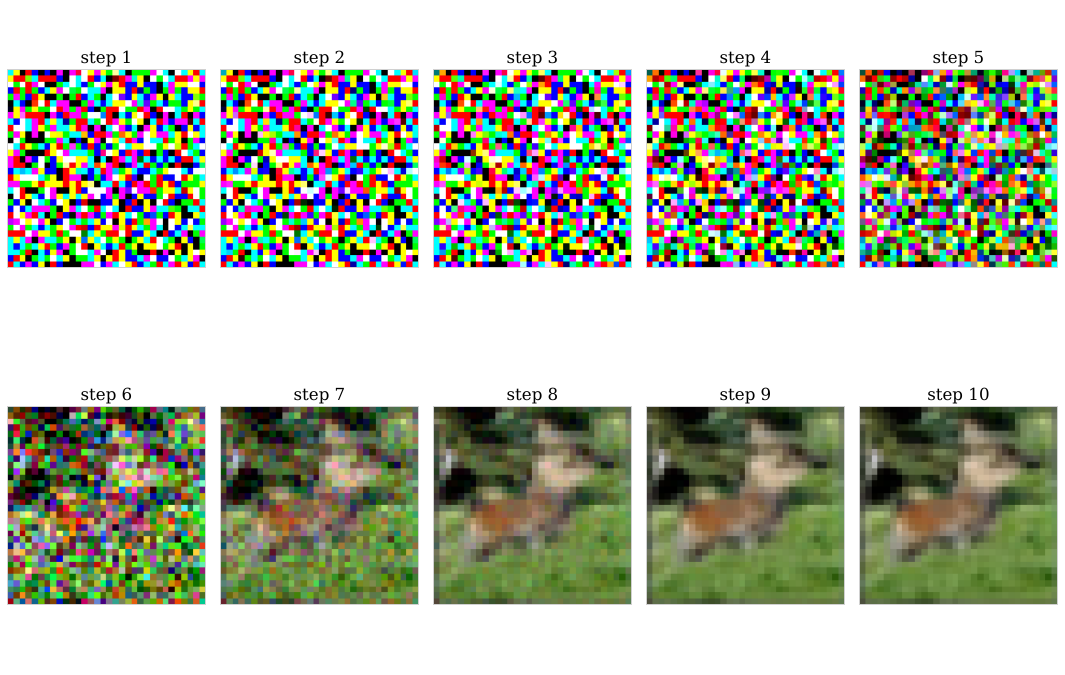}
    \caption{Cond--marg log1p.}
  \end{subfigure}
  \caption{EDM ten-step ODE-Heun trajectory grids. The quantitative five-step claim is carried by \Cref{tab:edm-fid-ode}; Log1p allocation is visually close to linear while remaining a distinct measured grid.}
  \label{fig:edm_images}
\end{figure}



\begin{figure}[H]
  \centering
  \includegraphics[width=0.82\linewidth]{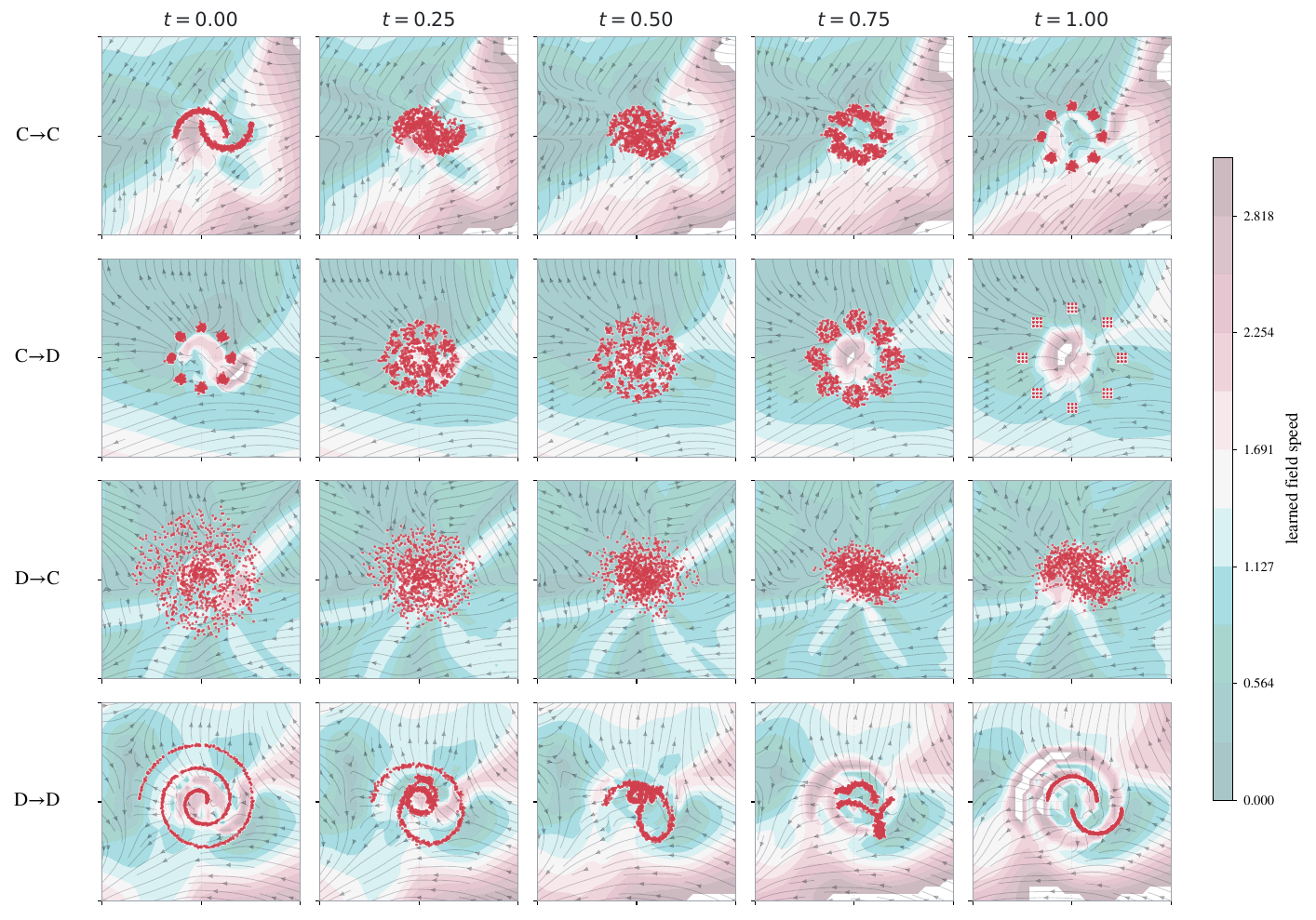}
  \caption{Synthetic transport scenarios used as controlled probes. The grid contrasts continuous--continuous, continuous--discrete, discrete--continuous, and discrete--discrete endpoint structure. The point of this panel is to show that the bridge path can face different endpoint constraints even before any high-dimensional neural sampler is introduced..}
  \label{fig:tarchic_rose_points_appendix}
\end{figure}

\begin{figure}[H]
  \centering
  \includegraphics[width=0.78\linewidth]{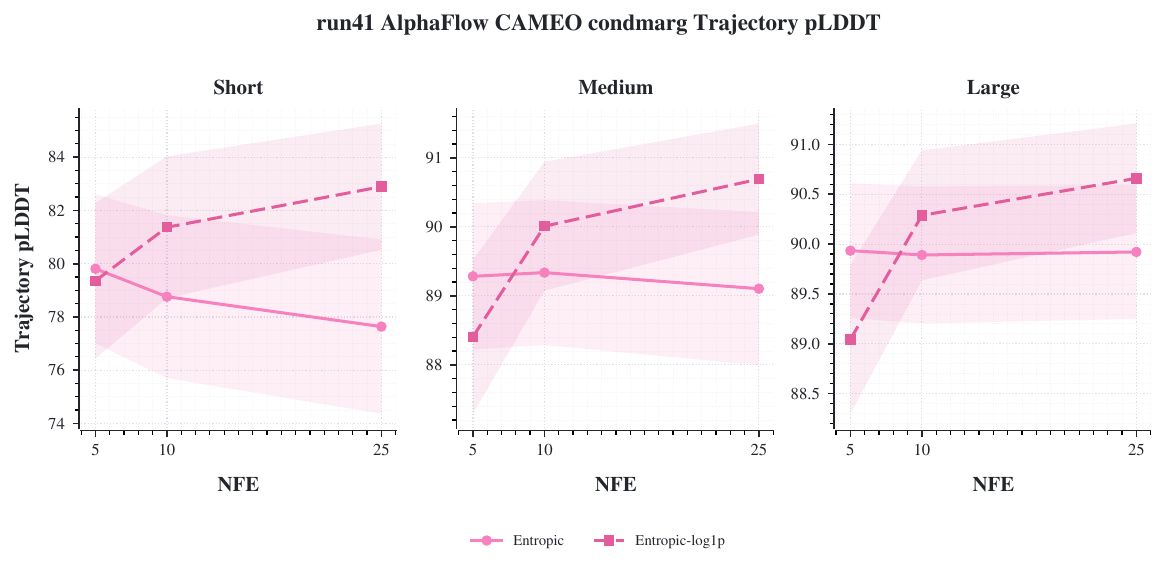}
  \caption{CAMEO AlphaFlow trajectory pLDDT by dataset size. This panel documents endpoint-confidence trends across the trajectory and does not replace the cond--marg schedule diagnostic.}
  \label{fig:cameo_alphaflow_plddt_appendix}
\end{figure}

\begin{figure}[H]
  \centering
  \includegraphics[width=0.98\linewidth]{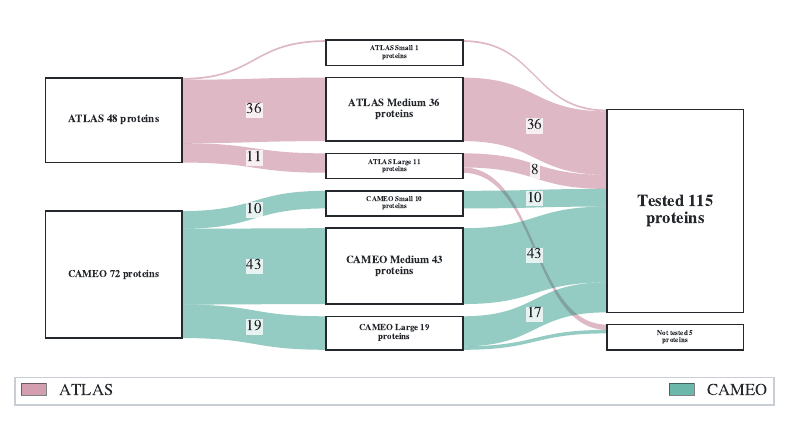}
  \caption{AlphaFlow protein-testing flow summary. This figure records the filtering and evaluation path behind the protein evidence.}
  \label{fig:alphaflow_testing_sankey_appendix}
\end{figure}



\begin{figure}[H]
  \centering
  \includegraphics[width=0.82\linewidth]{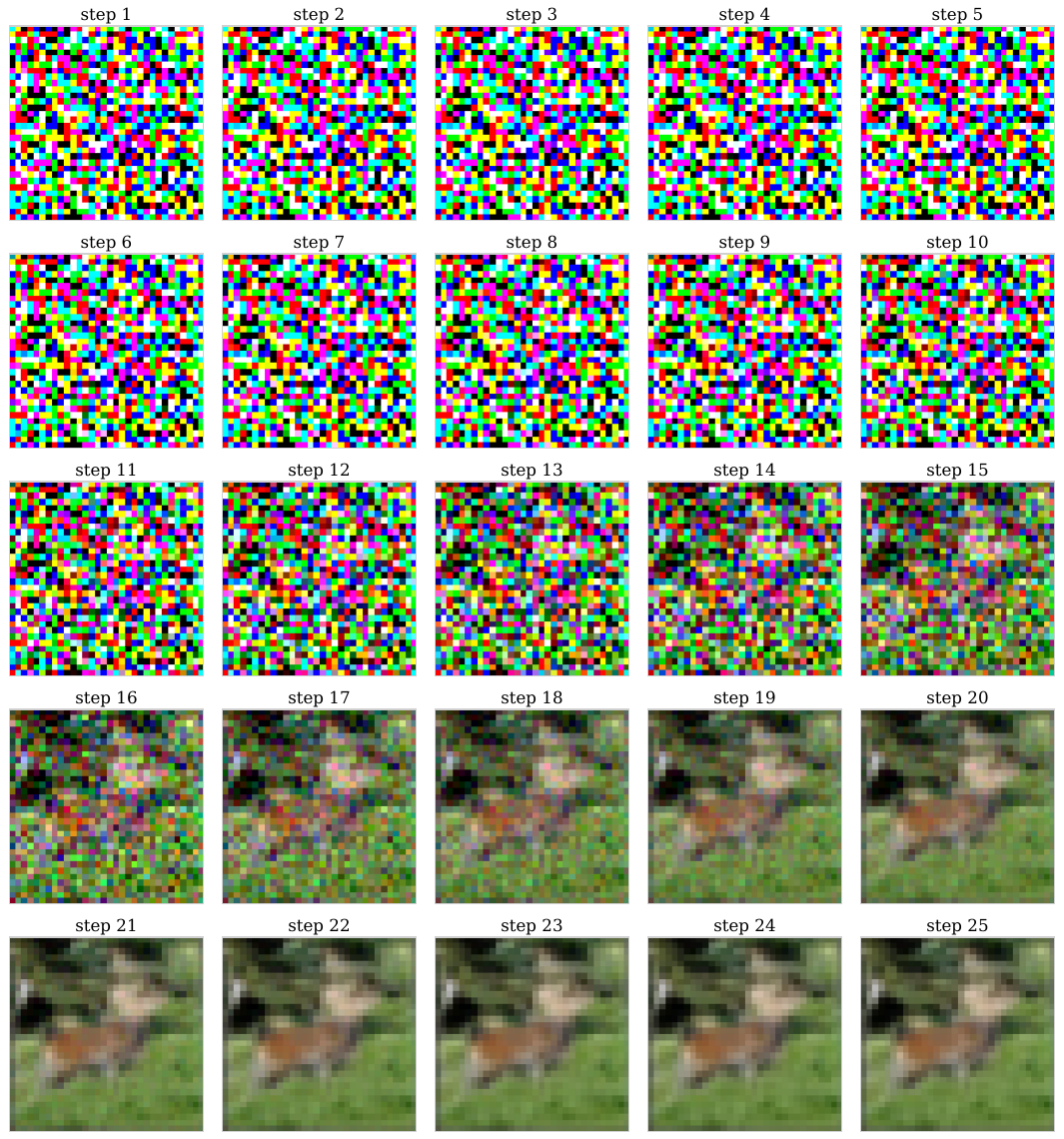}
  \caption{EDM ODE-Heun trajectory grid at 25 steps. The reader should observe that well-placed schedules begin to converge in visual behavior, matching the quantitative convergence in FID.}
  \label{fig:edm_ode_25_appendix}
\end{figure}

\begin{figure}[H]
  \centering
  \begin{subfigure}[t]{0.48\linewidth}
    \centering
    \includegraphics[width=\linewidth]{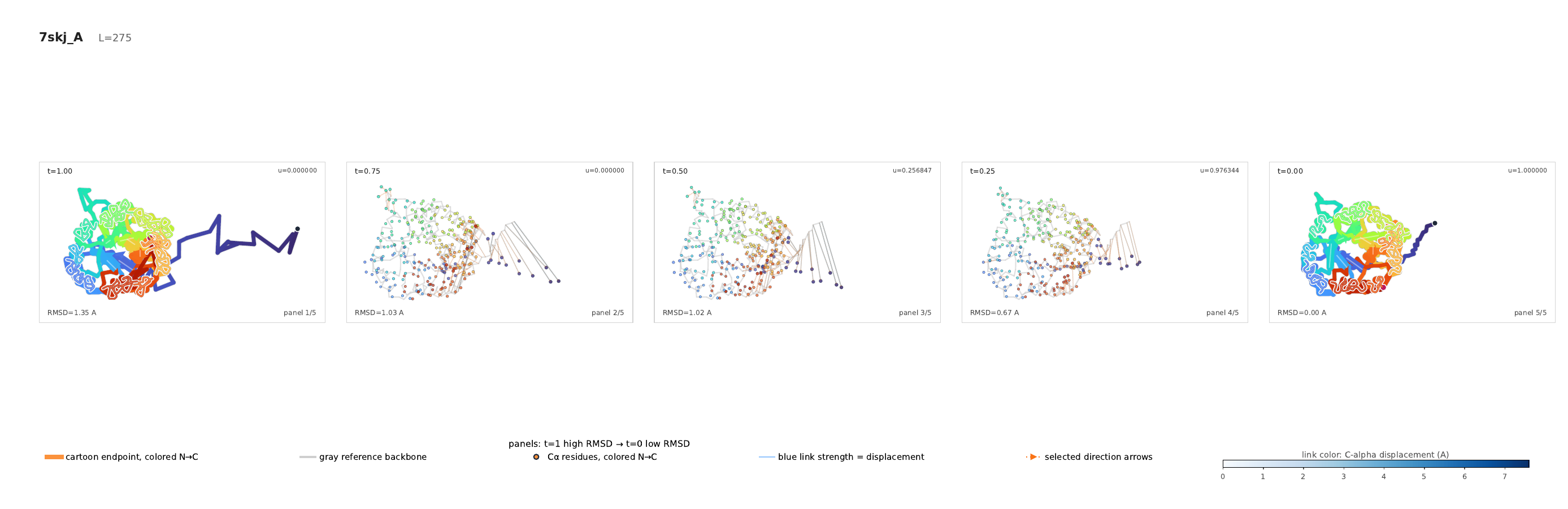}
    \caption{2D view.}
  \end{subfigure}
  \hfill
  \begin{subfigure}[t]{0.48\linewidth}
    \centering
    \includegraphics[width=\linewidth]{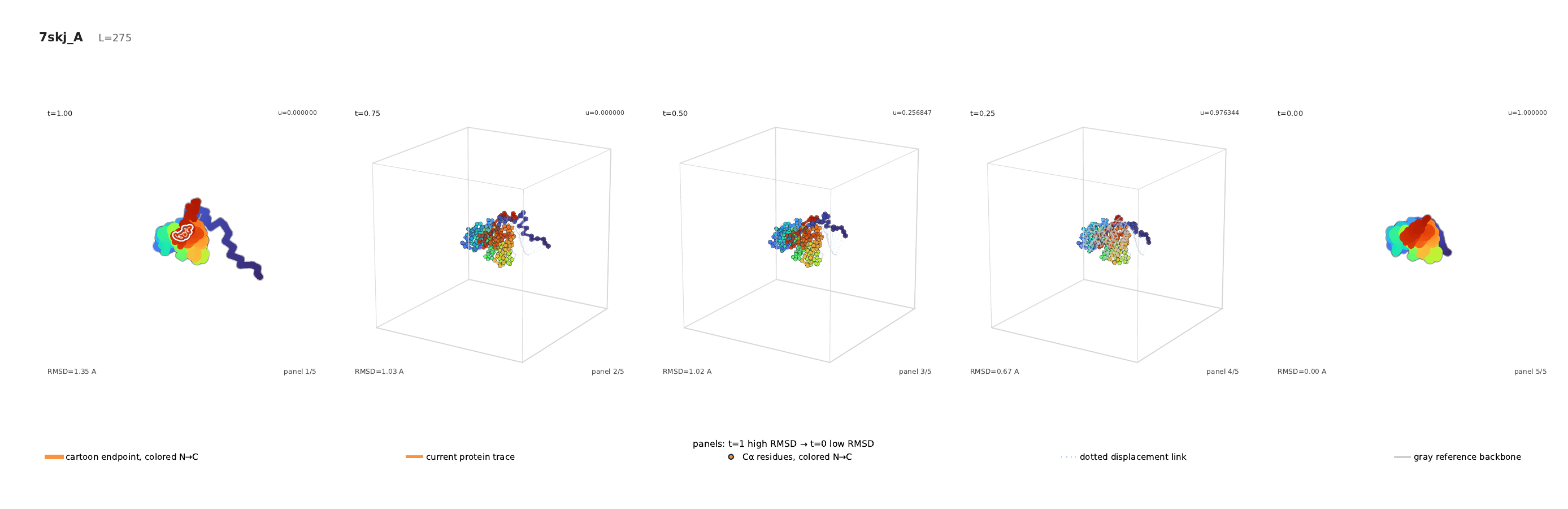}
    \caption{3D view.}
  \end{subfigure}
  \caption{AlphaFlow cond--marg protein evolution at 5 steps for the same target. The paired 2D and 3D renderings are trajectory diagnostics, not endpoint-quality metrics.}
  \label{fig:alphaflow_evolution_5steps_appendix}
\end{figure}

\begin{figure}[H]
  \centering
  \begin{subfigure}[t]{0.48\linewidth}
    \centering
    \includegraphics[width=\linewidth]{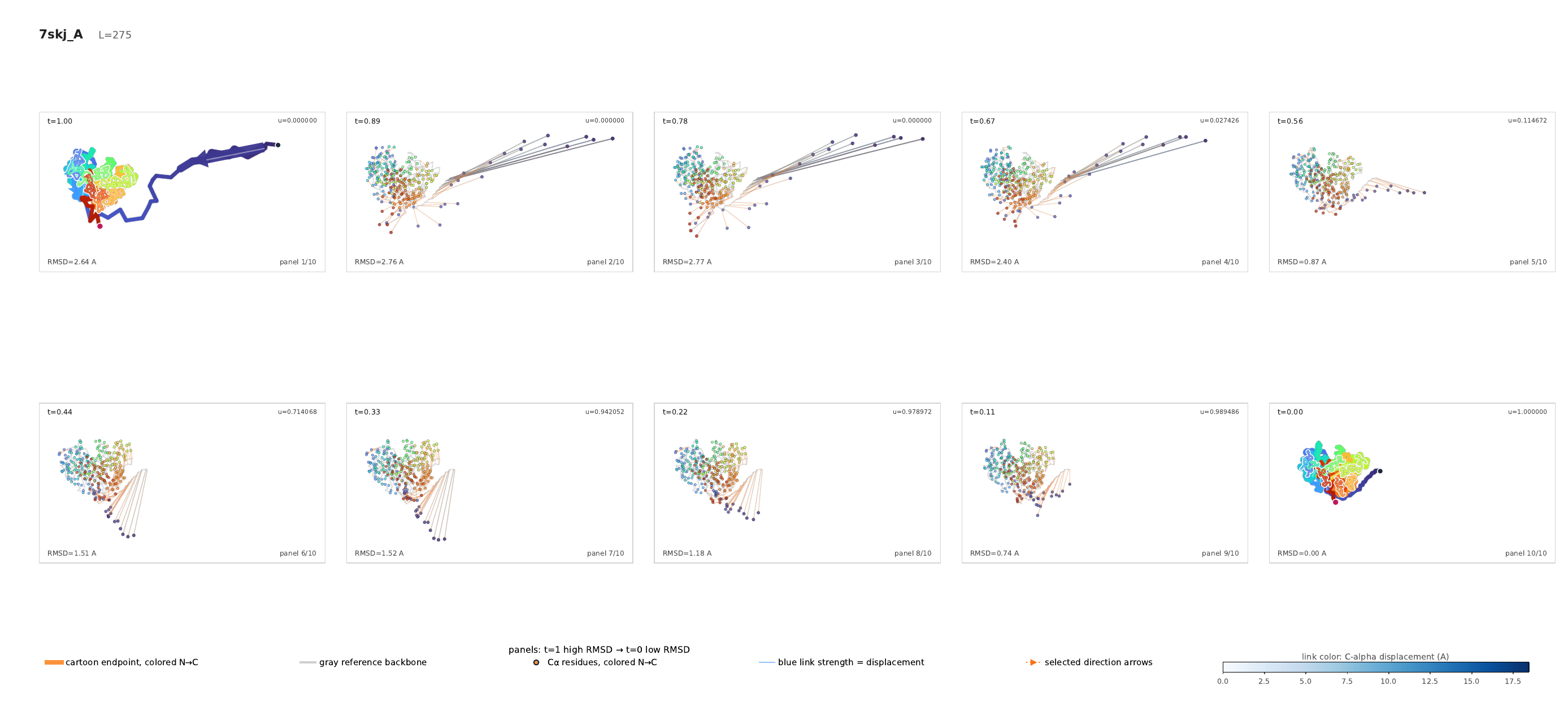}
    \caption{2D view.}
  \end{subfigure}
  \hfill
  \begin{subfigure}[t]{0.48\linewidth}
    \centering
    \includegraphics[width=\linewidth]{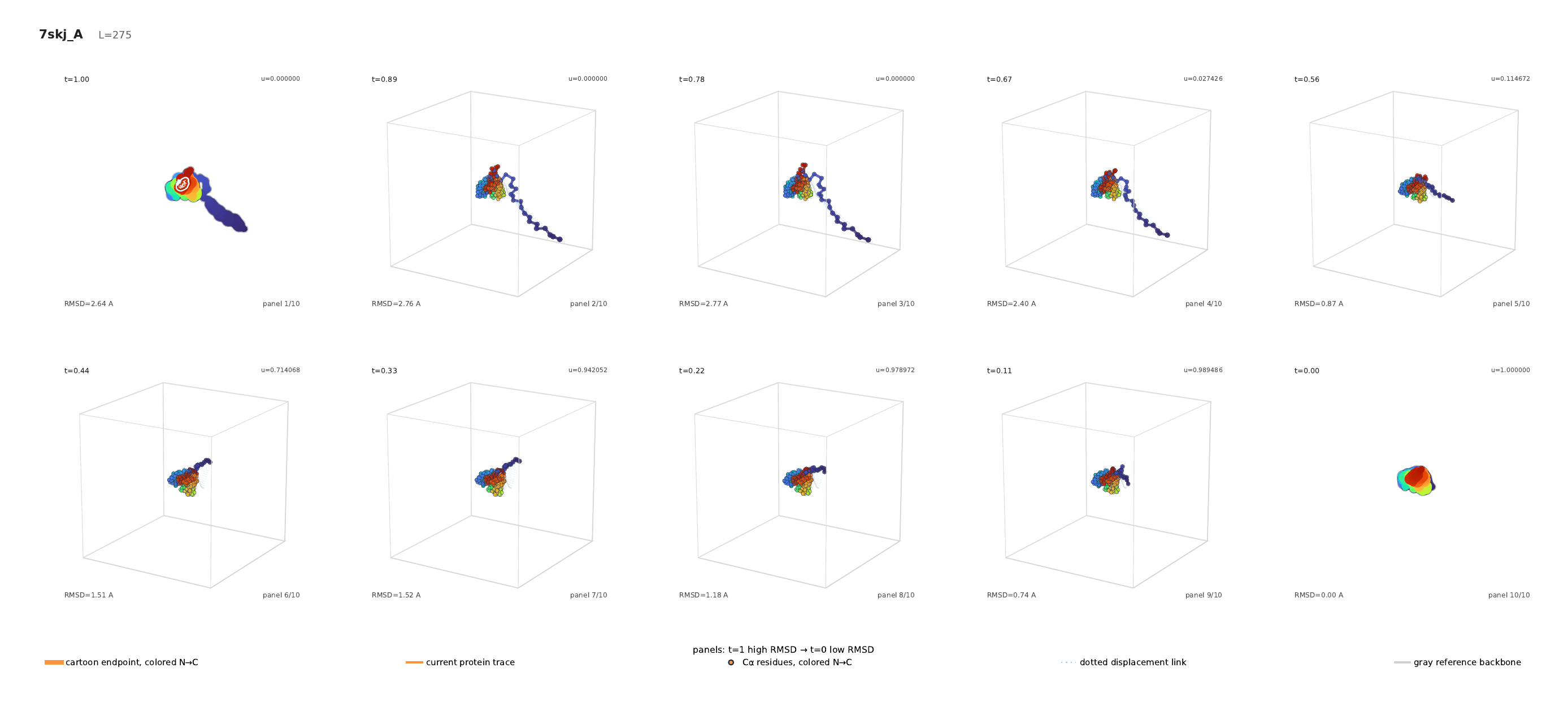}
    \caption{3D view.}
  \end{subfigure}
  \caption{AlphaFlow cond--marg protein evolution at 10 steps for the same target. The paired views make the trajectory easier to inspect without changing the endpoint metric interpretation.}
  \label{fig:alphaflow_evolution_10steps_appendix}
\end{figure}


\begin{figure}[H]
  \centering
  \begin{subfigure}[t]{0.48\linewidth}
    \centering
    \includegraphics[width=\linewidth]{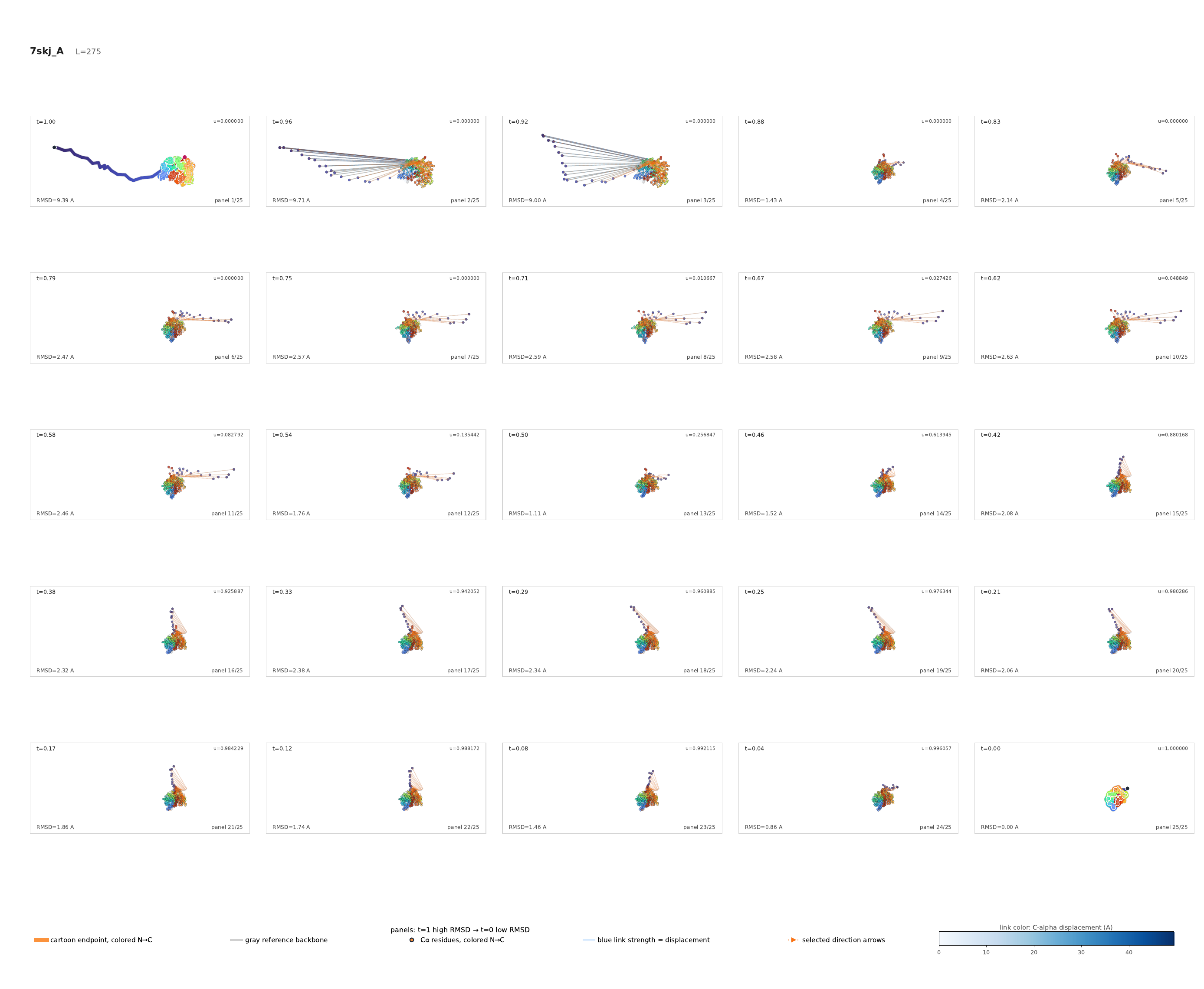}
    \caption{2D view.}
  \end{subfigure}
  \hfill
  \begin{subfigure}[t]{0.48\linewidth}
    \centering
    \includegraphics[width=\linewidth]{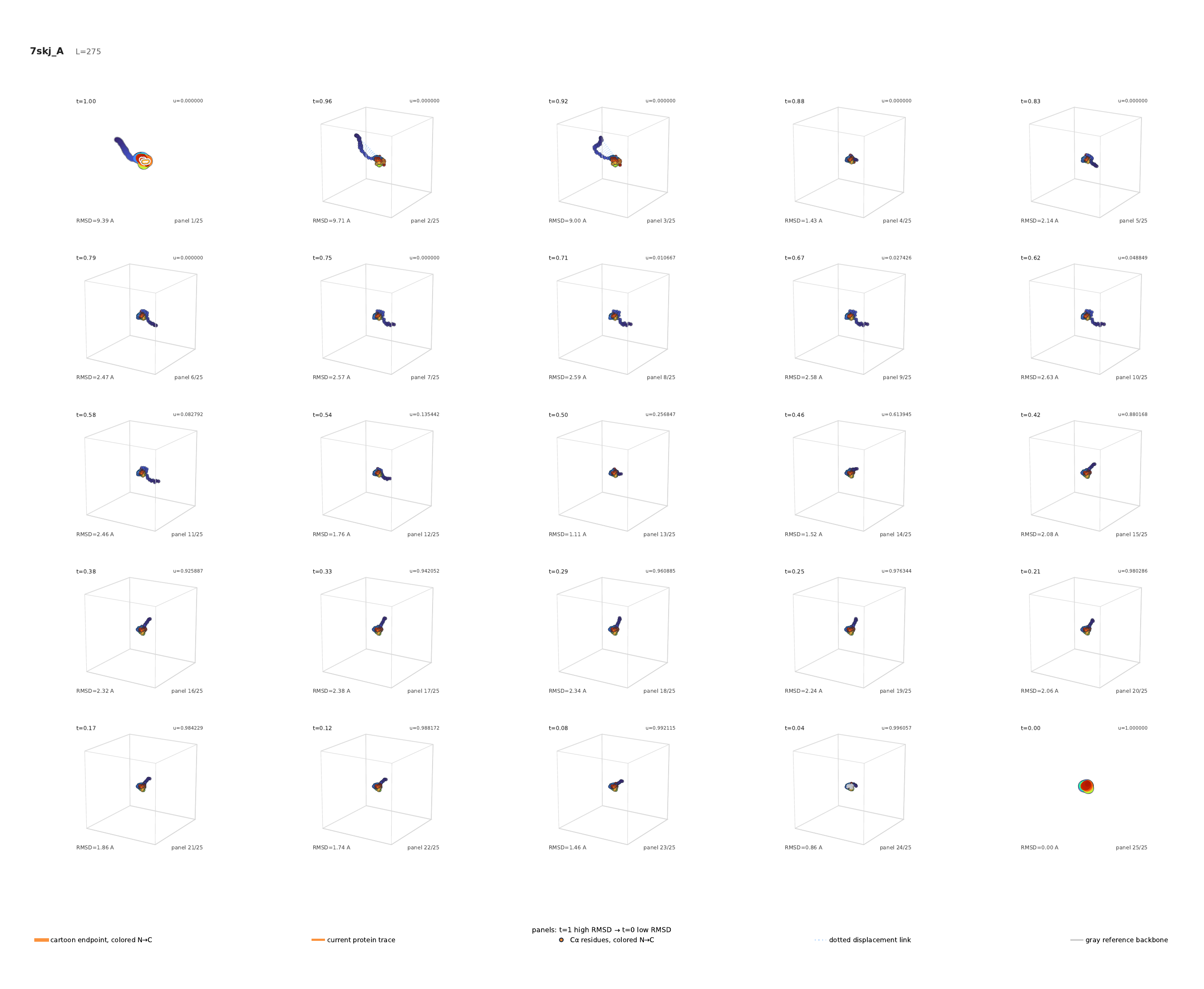}
    \caption{3D view.}
  \end{subfigure}
  \caption{AlphaFlow cond--marg protein evolution at 25 steps for the same target. This higher-NFE view supports trajectory inspection and remains separate from quantitative pLDDT claims.}
  \label{fig:alphaflow_evolution_25steps_appendix}
\end{figure}

\section{Computational Notes}
\label{app:computational}
\paragraph{The inference data for our test runs was collected on a SLURM-managed HPC GPU cluster using NVIDIA H100 80GB HBM3 GPUs}. The main inference sweeps ran as exclusive 8-GPU node jobs, typically requesting gpus-per-node=8 with torchrun-nproc\_per\_node=8, about 80 requested CPU cores per job, 128 CPU cores allocated by SLURM, and roughly 2 TB of node memory. Our runs also included several smaller 1-GPU AlphaFlow precompute jobs with 16 CPUs and about 250 GB allocated memory. Counting allocated GPU wall time, including failed and timed-out retry attempts because they still consumed resources, and excluding CPU-only postprocessing, our runs consumed about 333.67 H100 GPU-hours in total: 183.52 GPU-hours for EDM and 150.15 GPU-hours for AlphaFlow, including experimentation. If the earlier invalid or cancelled launches are included as an all-in accounting number, the estimate rises slightly to about 334.42 H100 GPU-ours.

\paragraph{Computational cost.} Our method is computationally light relative to both EDM and AlphaFlow because it is a scheduler rather than a new generative backbone. At a fixed number of function evaluations (NFE), it uses the same trained-model calls as the corresponding baseline linear, cosine, sigmoid, or power schedule, and only changes the time grid on which those evaluations are placed. The additional online overhead consists of reading or interpolating a precomputed entropy-rate curve and constructing the schedule, which is $\mathcal{O}(K)$ for $K$ sampling steps and negligible compared with a single neural-network evaluation. The offline entropy-curve estimation incurs a one-time cost proportional to the number of time points and probe samples, but this cost is amortized across all subsequent samples and does not change the asymptotic inference complexity.

For EDM, inference is dominated by denoising-network evaluations. If $B$ is the batch size, $K$ is the number of model evaluations, and $C_{\mathrm{EDM}}$ is the cost of one EDM U-Net evaluation at the target image resolution, then sampling costs approximately
\[
\mathcal{O}(B K C_{\mathrm{EDM}}),
\]
up to constant-factor differences between Euler, Heun, and SDE-corrector implementations. Our entropic schedule has the same complexity at matched NFE; any improvement comes from allocating the same evaluations more effectively, or from reaching a target quality with fewer evaluations.

For AlphaFlow, the dominant cost is substantially larger because each denoising or integration step invokes a protein-structure model, here the ESMFold-based AlphaFlow checkpoint. For a protein of length $L$, each model call scales at least quadratically in $L$ due to attention and pairwise geometric representations, with additional architecture-dependent costs in the structure module. AlphaFlow sampling therefore scales roughly as
\[
\mathcal{O}(B K C_{\mathrm{AF}}(L)),
\]
where $C_{\mathrm{AF}}(L)$ grows steeply with protein length and memory also increases strongly with $L$. Our scheduler does not introduce additional AlphaFlow passes; it only changes the temporal allocation of the same 5-, 10-, or 25-step budgets.



\end{document}